%% file: neurips2026.tex
\PassOptionsToPackage{numbers,sort&compress}{natbib}
\documentclass{article}
\usepackage[preprint]{neurips_2026}

\usepackage[utf8]{inputenc}
\usepackage[T1]{fontenc}
\usepackage{microtype}
\usepackage{graphicx}
\usepackage{booktabs}
\usepackage{array}
\usepackage{tabularx}
\usepackage{float}
\usepackage{amsmath}
\usepackage{amssymb}
\usepackage{amsthm}
\usepackage{xcolor}
\usepackage{tikz}
\usetikzlibrary{arrows.meta,positioning,fit,calc,backgrounds}
\usepackage{hyperref}
\usepackage[capitalize,noabbrev]{cleveref}

\makeatletter
\newcommand*{\proxyguard@tabletype}{table}
\AddToHook{cmd/@floatboxreset/after}{%
  \ifx\@captype\proxyguard@tabletype
    \setlength{\abovecaptionskip}{2pt}%
    \setlength{\belowcaptionskip}{0pt}%
    \setlength{\parskip}{0pt}%
  \fi
}
\makeatother
\definecolor{darkblue}{rgb}{0,0,0.5}
\hypersetup{
  colorlinks=true,
  citecolor=darkblue,
  linkcolor=darkblue,
  urlcolor=darkblue
}

\makeatletter
\renewcommand{\@noticestring}{}
\makeatother

\theoremstyle{plain}
\newtheorem{theorem}{Theorem}
\newtheorem{proposition}[theorem]{Proposition}
\crefname{theorem}{Theorem}{Theorems}
\crefname{proposition}{Proposition}{Propositions}
\theoremstyle{definition}
\newtheorem{definition}[theorem]{Definition}
\newtheorem{assumption}[theorem]{Assumption}
\crefname{assumption}{Assumption}{Assumptions}
\theoremstyle{remark}

\newcommand{\method}{ProxyGuard}
\newcommand{\appref}[1]{Appendix~\ref{#1}}

\title{ProxyGuard: Direct Reliability Inference for\\
Randomized Data Release Mechanisms with Shared Targets}

\author{Dipesh Tharu Mahato\textsuperscript{*} \\
New York University \\[-1pt]
{\small\texttt{dm6259@nyu.edu}}
\And
Pramod Dhungana \\
Queens College \\[-1pt]
{\small\texttt{pramodkumar.dhungana97@qmail.cuny.edu}}}

\begin{document}

\maketitle
\vspace{-0.8em}
\footnotetext{\textsuperscript{*}Corresponding author:
\texttt{dm6259@nyu.edu}. Code and experiment scripts are available at
\url{https://github.com/dipeshbabu/ProxyGuard}.}

\begin{abstract}
Researchers often choose a proxy dataset from many releases, transformations,
or seeds. Search can make an invalid release appear adequate, while one
adequate release does not establish that its generator is reliable. \method{}
controls both errors using prespecified bounded risks and a sealed target set.
Named-release mode corrects for multiplicity and certifies specific releases.
Direct shared-target mode evaluates independent mechanism draws on a common
target, lower-bounds their favorable-score rate, and subtracts a bound on
favorable scores contributed by invalid releases. Conditional on the target,
release scores are independent, yielding a finite-sample mechanism-reliability
guarantee without independent target batches or assumptions on release-level
$p$-value dependence. We show that the mean-only penalty is sharp and derive
a smooth-score certificate with additive target concentration. In a registered three-requirement
study, direct mode raises power from 5.6\% to 64.2\% at reliability 0.95, while
named mode remains stronger under high-signal evidence. Prospective audits span
full-pipeline Rice--TVAE, which retrains on every draw, and a non-tabular text
mechanism.
\end{abstract}

\input{sections/introduction}
\input{sections/related_work}
\input{sections/setup}
\input{sections/method}
\input{sections/experiments}
\input{sections/limitations}
\input{sections/conclusion}

\clearpage
\bibliographystyle{unsrtnat}
\bibliography{bibli}

\clearpage
\appendix
\input{sections/appendix}

\end{document}

%% file: sections/introduction.tex
\section{Introduction}
\label{sec:introduction}

A research group may be unable to share the table used for model development.
It may release synthetic data, an older cohort, or a public stand-in instead.
The receiving group then needs to know whether a workflow developed on that
proxy still supports a specified decision on the target population. Standard
synthetic-data evaluations provide useful measurements, but they do not by
themselves justify this claim \citep{alaa2022faithful,lautrup2025syntheval}.

Two problems arise. An evaluator may screen many generators, privacy settings,
transformations, or seeds and retain the release with the most favorable
observed results. A point threshold applied after that search can accept an
invalid proxy by chance. Even when the selected release is adequate, it may be
an atypically favorable draw from an unreliable
generator. Evidence about one table does not say how often the frozen
mechanism will succeed again.

The downstream specification matters as well. Similar AUC values do not imply
similar calibration, decision cost, or inferential behavior
\citep{vanbreugel2023realerrors,decruyenaere2024inferential}. We therefore
define proxy validity through losses and limits that researchers fix before
opening the target data. A relative requirement compares proxy- and
source-trained procedures on the same target records. An absolute requirement
limits the proxy procedure's target risk. An audit may require either form or
both.

\method{} separates two inferential tasks. Named-release mode tests every
registered requirement on a sealed target set and applies Holm correction
across releases \citep{holm1979simple,angelopoulos2021ltt}. It can identify
particular releases, but an unresolved release adds no positive evidence to a
mechanism claim. Direct shared-target mode answers only the mechanism
question. It scores independent mechanism draws on the same target records
and estimates how often they score favorably. Because a finite target can
make an invalid release appear favorable, direct mode subtracts a false-pass
allowance. The remainder lower-bounds the probability that a fresh release is
truly valid.

\paragraph{Contributions.}
Our contributions are threefold. \textbf{First,} we formulate mechanism-level
reliability when independent randomized releases share one random target. We
derive a simultaneous finite-sample lower bound that separates mechanism-draw
uncertainty from target-induced false-pass contamination, without independent
target batches or assumptions on release-level $p$-value dependence.
\textbf{Second,} we characterize expectation-only contamination control: the
mean false-pass bound permits no uniformly smaller target-side correction. We
then add bounded-sensitivity structure through registered ramp scores and
derive an additive target-concentration certificate. \textbf{Third,} we
characterize empirically when direct mechanism inference and named-release
certification use evidence differently. Registered simulations identify
complementary moderate- and high-signal regimes, while prospective audits show
applicability to full-pipeline neural retraining and a non-tabular mechanism.
Named-release inference, stratification, planning, and provenance support
these contributions rather than add separate novelties. Unlike fixed-configuration risk control,
partial-conjunction counting, and fixed-error reliability demonstration, the
direct certificate controls a target-random false-pass contribution shared
across release evaluations.

At reliability 0.95, direct mode has 64.2\% power versus 5.6\% for named in
the moderate-evidence setting; high-signal evidence reverses the ranking.
Prospective audits cover full-pipeline neural retraining, including Rice--TVAE
\citep{cinar2019rice,xu2019modeling}, and a frozen 20 Newsgroups text mechanism
\citep{lang1995newsweeder}.

\Cref{fig:protocol} summarizes the named-release and direct finite-sample
workflows. Development ends before the sealed target audit begins.

\begin{figure}[H]
\centering
\resizebox{\textwidth}{!}{\input{figs/proxyguard_protocol.tikz}}
\caption{\method{} workflow. Development ends before the target audit opens.
Named-release mode identifies particular releases; direct shared-target mode
uses favorable-score frequency after correcting for invalid false passes.}
\label{fig:protocol}
\end{figure}
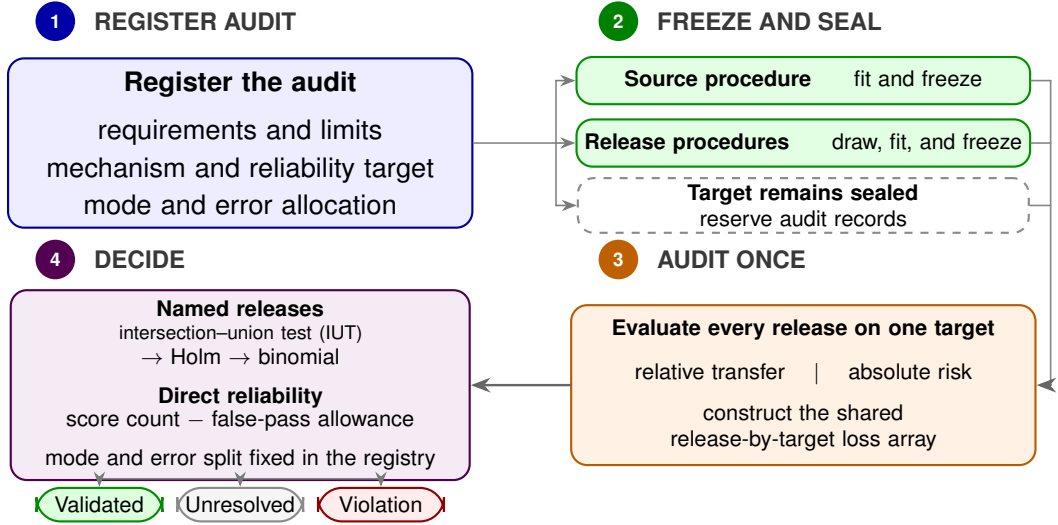

%% file: figs/proxyguard_protocol.tikz
\begin{tikzpicture}[
  font=\sffamily\normalsize,
  card/.style={
    rectangle,
    rounded corners=6pt,
    draw,
    line width=0.8pt,
    align=center,
    inner sep=5pt
  },
  claim/.style={
    card,
    draw=blue!62!black,
    fill=blue!7,
    text width=5.75cm,
    minimum height=1.60cm
  },
  pipeline/.style={
    card,
    draw=green!57!black,
    fill=green!11,
    text width=5.75cm,
    minimum height=0.62cm,
    inner sep=3pt,
    font=\sffamily\footnotesize
  },
  sealed/.style={
    card,
    dashed,
    draw=black!46,
    fill=white,
    text width=5.75cm,
    minimum height=0.48cm,
    inner sep=2.5pt,
    font=\sffamily\footnotesize
  },
  audit/.style={
    card,
    draw=orange!72!black,
    fill=orange!11,
    text width=5.75cm,
    minimum height=2.10cm,
    font=\sffamily\footnotesize
  },
  decision/.style={
    card,
    draw=violet!62!black,
    fill=violet!8,
    text width=5.75cm,
    minimum height=2.10cm,
    inner sep=4pt,
    font=\sffamily\footnotesize
  },
  outcome/.style={
    rectangle,
    rounded corners=9pt,
    draw,
    line width=0.8pt,
    align=center,
    text width=1.55cm,
    minimum height=0.44cm,
    inner sep=1.5pt,
    font=\sffamily\footnotesize
  },
  stepnum/.style={
    circle,
    text=white,
    minimum size=0.52cm,
    inner sep=0pt,
    font=\sffamily\bfseries\scriptsize
  },
  steptitle/.style={
    font=\sffamily\bfseries\small,
    anchor=west,
    text=black!78
  },
  arrow/.style={
    -{Stealth[length=2.6mm]},
    line width=0.9pt,
    draw=black!60,
    line cap=round,
    line join=round
  },
  innerarrow/.style={
    -{Stealth[length=2.0mm]},
    semithick,
    draw=black!52,
    line cap=round,
    line join=round
  },
  bus/.style={
    semithick,
    draw=black!48,
    line cap=round,
    line join=round
  }
]

\node[claim] (claimscope) at (-3.70,1.43) {
  \textbf{Register the audit}\\[7pt]
  requirements and limits\\[3pt]
  mechanism and reliability target\\[3pt]
  mode and error allocation
};

\node[pipeline] (sourcefit) at (3.70,2.25) {
  \textbf{Source procedure}\qquad fit and freeze
};
\node[pipeline] (proxyfit) at (3.70,1.43) {
  \textbf{Release procedures}\qquad draw, fit, and freeze
};
\node[sealed] (target) at (3.70,0.61) {
  \textbf{Target remains sealed}\qquad reserve audit records
};

\coordinate (devtop)    at (0.45,2.25);
\coordinate (devmiddle) at (0.45,1.43);
\coordinate (devbottom) at (0.45,0.61);
\draw[bus] (claimscope.east) -- (devmiddle);
\draw[bus] (devtop) -- (devbottom);
\draw[innerarrow] (devtop) -- (sourcefit.west);
\draw[innerarrow] (devmiddle) -- (proxyfit.west);
\draw[innerarrow] (devbottom) -- (target.west);

\node[audit] (paired) at (3.70,-1.75) {
  \textbf{Evaluate every release on one target}\\[7pt]
  relative transfer \quad $\vert$ \quad absolute risk\\[6pt]
  construct the shared release-by-target loss array
};

\node[decision] (family) at (-3.70,-1.75) {
  \textbf{Named releases}\\[-1pt]
  {\scriptsize intersection--union test (IUT)} $\rightarrow$ Holm $\rightarrow$ binomial\\[5pt]
  \textbf{Direct reliability}\\[-1pt]
  score count $-$ false-pass allowance\\[5pt]
  mode and error split fixed in the registry
};

\node[outcome,draw=green!55!black,fill=green!13]
  (validated) at (-5.55,-3.32) {Validated};
\node[outcome,draw=black!46,fill=black!4]
  (unresolved) at (-3.70,-3.32) {Unresolved};
\node[outcome,draw=red!58!black,fill=red!7]
  (violation) at (-1.85,-3.32) {Violation};

\coordinate (busleft)  at (-5.55,-2.98);
\coordinate (busmid)   at (-3.70,-2.98);
\coordinate (busright) at (-1.85,-2.98);
\coordinate (busdrop)  at (-3.70,-2.98);
\draw[bus] (family.south) -- (busdrop);
\draw[bus] (busleft) -- (busright);
\draw[innerarrow] (busleft) -- (validated.north);
\draw[innerarrow] (busmid) -- (unresolved.north);
\draw[innerarrow] (busright) -- (violation.north);

\coordinate (auditbusTop)    at (6.95,2.25);
\coordinate (auditbusMiddle) at (6.95,1.43);
\coordinate (auditbusTarget) at (6.95,0.61);
\coordinate (auditbusEntry)  at (6.95,-1.75);
\draw[bus] (sourcefit.east) -- (auditbusTop);
\draw[bus] (proxyfit.east) -- (auditbusMiddle);
\draw[bus] (target.east) -- (auditbusTarget);
\draw[bus] (auditbusTop) -- (auditbusEntry);
\draw[arrow] (auditbusEntry) -- (paired.east);
\draw[arrow] (paired.west) -- (family.east);

\node[stepnum,fill=blue!65!black] (step1) at (-6.13,3.03) {1};
\node[steptitle] at ($(step1.east)+(0.12cm,0)$) {REGISTER AUDIT};

\node[stepnum,fill=green!50!black] (step2) at (1.27,3.03) {2};
\node[steptitle] at ($(step2.east)+(0.12cm,0)$) {FREEZE AND SEAL};

\node[stepnum,fill=orange!75!black] (step3) at (1.27,-0.10) {3};
\node[steptitle] at ($(step3.east)+(0.12cm,0)$) {AUDIT ONCE};

\node[stepnum,fill=violet!65!black] (step4) at (-6.13,-0.10) {4};
\node[steptitle] at ($(step4.east)+(0.12cm,0)$) {DECIDE};

\end{tikzpicture}

%% file: sections/related_work.tex
\section{Related work}
\label{sec:related_work}

Synthetic-data evaluations separate fidelity, downstream utility, and
empirical privacy \citep{alaa2022faithful,lautrup2025syntheval}. These
measurements can reveal failures, but they do not control the error incurred
after searching over candidate releases. Learn-then-Test converts risk
requirements into simultaneous tests for predictive configurations
\citep{angelopoulos2021ltt}. \method{} applies this logic to the learning
procedure induced by a proxy, then adds a reliability question over future
draws from its release mechanism.

Classical partial-conjunction tests can lower-bound how many releases are
valid without identifying them \citep{benjamini2008screening}. Tail-Simes
needs independence or positive regression dependence on a subset (PRDS),
while tail-Fisher needs independence; conditional
variants gain power under further assumptions \citep{liang2025conditional}.

These methods do not by themselves justify reusing one target sample across
all release tests. Our direct result instead conditions on that shared target,
infers the mean release score, and controls the invalid-release contribution
through a separate target-contamination event. Unlike imperfect-inspection
models, the false-pass contribution here is target-random and shared across
release evaluations rather than a fixed error rate. The named-release mode
retains classical binomial reliability inference \citep{clopper1934confidence}.
The methods differ in their estimands and dependence assumptions. Tail-Simes uses independence or
PRDS to lower-bound a count; imperfect-inspection models fix the error model;
hierarchical generalized linear mixed models posit a random-effects
distribution \citep{breslow1993approximate}. Named \method{} identifies releases
on a shared random target, whereas direct \method{} gives up their identities
for a finite-sample reliability bound.
\appref{app:related_work} reviews connections to reliability demonstration,
imperfect inspection, and generator evaluation.

%% file: sections/setup.tex
\section{What it means for a proxy to support a claim}
\label{sec:setup}

\subsection{Source, proxy, and target}
\label{sec:source_proxy_target}

Let $\mathcal{Q}=\{Q_1,\ldots,Q_M\}$ be a finite set of candidate proxies.
Candidate $Q$ has a declared source table $P_Q$ and target distribution
$T_Q$. Candidates may share both, as when several synthesizers replace one
table, or may belong to different tasks in one correction family. A proxy
may be synthetic, transformed, drawn from another site, or taken from
another time period.

A learning procedure $S$ covers model fitting, calibration, and decision
threshold selection. It produces $f_{P_Q}=S(P_Q)$ and $f_Q=S(Q)$. The
procedure may select from a fixed model library using development data. It
must finish before researchers inspect target audit outcomes.

\begin{assumption}[Independent target audit]
\label{ass:audit}
Researchers register the audit design before target access and draw an i.i.d.\
sample from $T_Q$. The design keeps the complete audit sample independent of
proxy construction, model fitting, calibration, threshold selection, and the
choice of requirements. Researchers may share audit records across releases
and mechanisms.
\end{assumption}

A holdout used to tune a proxy is development data and cannot also justify the
final claim. We introduce fixed-size stratified designs after the core direct
result in \cref{sec:direct_shared_target}.

\subsection{Registered requirements}
\label{sec:registered_requirements}

Let $\mathcal{J}$ index the claims in the audit. A relative-transfer
requirement compares proxy and source procedures on the same target record:
\begin{equation}
\label{eq:paired_regret}
D_{i,Q,j}
=
\ell_j(f_Q,Z_{i,Q})-\ell_j(f_{P_Q},Z_{i,Q}),
\qquad
\Delta_{Q,j}=\mathbb{E}_{T_Q}[D_{i,Q,j}].
\end{equation}
This pairing often removes variation that would appear in two test samples.
Relative transfer alone does not establish that the proxy procedure is
adequate. An absolute-risk requirement instead uses
\[
A_{i,Q,j}=\ell_j(f_Q,Z_{i,Q}),
\qquad
R_{Q,j}=\mathbb E_{T_Q}[A_{i,Q,j}].
\]
An application may register either form or both. Write $X_{i,Q,j}$ for the
chosen per-record quantity, $\mu_{Q,j}$ for its mean, and $\tau_j$ for its
limit. Thus $(X,\mu,\tau)$ is $(D,\Delta,\epsilon)$ for relative transfer
and $(A,R,r)$ for absolute adequacy. The registry must give known bounds
$a_j\leq X_{i,Q,j}\leq b_j$. A loss in $[0,1]$ gives relative-regret bounds
$[-1,1]$ and absolute-risk bounds $[0,1]$.

\begin{definition}[Claim-valid proxy]
\label{def:valid}
Candidate $Q$ is valid for the registered specification when
\begin{equation}
\label{eq:validity}
\mu_{Q,j}<\tau_j
\quad\text{for every }j\in\mathcal{J}.
\end{equation}
\end{definition}

\paragraph{What \method{} does not claim.}
\method{} certifies only the registered losses, limits, and target
distribution. A relative requirement does not establish absolute adequacy,
and an absolute requirement does not establish transfer. Neither result
establishes distributional similarity, causal transport, performance in
unregistered subgroups, or inferential validity for arbitrary downstream
analyses. Researchers must register any additional claim separately and use a
statistic suited to that claim.

\subsection{From one release to a release mechanism}
\label{sec:release_mechanism}

We use \emph{proxy} for the substitute dataset's role in the scientific claim
and \emph{release} for one realized draw from a randomized proxy mechanism.
Let $G_\theta$ denote a fixed release mechanism, including its generator,
privacy setting, preprocessing, and training procedure. For an independent release
$Q\sim G_\theta$, write
\[
A(Q)=
\mathbf{1}\!\left\{
\mu_{Q,j}<\tau_j
\text{ for every }j\in\mathcal J
\right\},
\qquad
\eta_\theta=\Pr_{Q\sim G_\theta}\{A(Q)=1\}.
\]
The claim conditions on the realized source table, registered target,
development split, model library, and frozen source procedure. Reliability
averages only over randomness registered inside $G_\theta$:
\begin{center}
\scriptsize
\setlength{\tabcolsep}{4pt}
\begin{tabularx}{0.96\textwidth}{@{}lX@{}}
\toprule
Registered mechanism choice & Meaning of the certificate \\
\midrule
dataset sampling only & conditional on one fitted generator \\
generator fitting and sampling & generator-training and sampling reliability \\
fitting, sampling, and proxy learning & full-pipeline reliability \\
fixed generator or learner seed & conditional on that development choice \\
hand-selected favorable seeds & does not support the mechanism claim \\
\bottomrule
\end{tabularx}
\end{center}
Every source of randomness to which the future-release claim should generalize
must be redrawn independently inside each mechanism draw; anything fixed is
conditioned upon. Independent draws therefore have independent $A(Q)$ even
when they share one target audit.

\begin{definition}[Reliable release mechanism]
\label{def:mechanism}
For a registered reliability target $\eta_0=1-\rho$, mechanism $G_\theta$ is
reliable when
\[
\eta_\theta>\eta_0.
\]
\end{definition}

This claim is stronger than validating one table, but does not cover a
mechanism retuned after the target audit.

Our audits use Brier loss, clipped normalized log loss, and normalized 5:1
false-negative cost. These lie in $[0,1]$, so paired regrets lie in $[-1,1]$;
the cost ratio is a registered sensitivity choice. AUC remains diagnostic,
because ranking transport needs a U-statistic bound. Subgroup claims require
separate registration and small groups often remain unresolved.

%% file: sections/method.tex
\section{From target losses to reliability}
\label{sec:method}

The two operating modes use different observed summaries of the same latent
release validity. \Cref{tab:inference_objects} fixes the terminology used
throughout this section. In particular, a favorable direct score is not a
release certificate.

\begin{table}[H]
\centering
\scriptsize
\setlength{\tabcolsep}{4pt}
\begin{tabularx}{\textwidth}{@{}lXcX@{}}
\toprule
Object & Meaning & Observed? & Role \\
\midrule
$A(Q)$ & true validity of release $Q$ & no & defines reliability \\
$p_Q$ & evidence against release invalidity & yes & named-release mode \\
$S_\theta(Q,Z)$ & favorable score on the shared target & yes & direct shared-target mode \\
$V_\theta^{\mathrm{Holm}}$ & individually recognized releases & yes & named mechanism bound \\
$\eta_\theta$ & probability that a fresh release is valid & no & mechanism estimand \\
\bottomrule
\end{tabularx}
\caption{Latent and observed objects. ``Recognized,'' ``scores favorably,''
and ``truly valid'' refer to different events.}
\label{tab:inference_objects}
\end{table}

\subsection{Decision language}
\label{sec:decision_language}

At release level, \emph{validated} means that the named-release procedure
rejects the bad-release null. \emph{Violation detected} means that a
separately corrected lower-bound test places at least one risk above its
limit. At mechanism level, the corresponding lower or upper reliability
bound must cross $\eta_0$. We call every other result \emph{unresolved};
failure to validate is not evidence of a violation.

For interpretation, we also report simultaneous requirement bounds, but they
do not replace these corrected decisions. This vocabulary separates three
questions that point estimates alone can blur: whether the data support
validity, whether they support a violation, or whether they support neither.

\subsection{Testing a named release}
\label{sec:named_release}

For candidate release $Q$ and requirement $j$, the bad-release null is
\begin{equation}
\label{eq:component_null}
H_{0,Q,j}:\mu_{Q,j}\geq\tau_j.
\end{equation}
Let $\bar X_{Q,j}$ and $\widehat V_{Q,j}$ be the sample mean and
Bessel-corrected variance on $n=n_Q$ target records. For $n\geq2$, the
default upper bound is
\begin{equation}
\label{eq:eb_bound}
U_{Q,j}(\delta)
=
\bar X_{Q,j}
+
\sqrt{\frac{2\widehat V_{Q,j}\log(2/\delta)}{n}}
+
\frac{7(b_j-a_j)\log(2/\delta)}{3(n-1)}.
\end{equation}
This empirical Bernstein bound follows \citet{maurer2009empirical}. Its
inversion gives
\begin{equation}
\label{eq:component_p}
p_{Q,j}=\inf\{\delta\in(0,1):U_{Q,j}(\delta)<\tau_j\},
\end{equation}
with $p_{Q,j}=1$ if the set is empty. We also provide a Hoeffding version.
\noindent\textbf{Confidence-bound inversion.}
If nonincreasing $U(\delta)$ satisfies
$\Pr_\mu\{U(\delta)<\mu\}\leq\delta$, let
$p=\inf\{\delta:U(\delta)<\tau\}$, with $p=1$ if the set is empty.
For $t<1$ and $\varepsilon\in(0,1-t)$, $p\leq t$ implies
$U(t+\varepsilon)<\tau\leq\mu$, so
$\Pr_\mu(p\leq t)\leq t+\varepsilon$.
Letting $\varepsilon\downarrow0$ proves super-uniformity; $t=1$ is immediate.
A release is invalid when at least one requirement fails. Validation therefore
requires evidence against every component null. The intersection--union test
uses
\begin{equation}
\label{eq:iut_p}
p_Q=\max_{j\in\mathcal J}p_{Q,j}.
\end{equation}
Thus a favorable Brier result cannot compensate for an unresolved cost claim.
We then apply Holm's procedure to the registered release family.

\begin{theorem}[False proxy-validation control]
\label{thm:fwer}
Suppose the target sample is independent of all development and release
choices, researchers fix the candidate releases and requirements before the
audit opens, and each $p_{Q,j}$ is super-uniform under
\cref{eq:component_null}. If $\widehat{\mathcal Q}$ is the set that Holm
rejects at level $\alpha$, then
\[
\Pr\!\left(
\exists Q\in\widehat{\mathcal Q},
\exists j\in\mathcal J:
\mu_{Q,j}\geq\tau_j
\right)
\leq\alpha.
\]
The result permits arbitrary dependence among requirements and releases.
\end{theorem}

The proof is in \appref{app:proofs}. Source and proxy procedures may select
different models, calibrators, or thresholds on development data. The theorem
conditions on those completed choices and an untouched target sample.

\subsection{Named-release mechanism inference}
\label{sec:named_mechanism}

Let $\Theta=\{\theta_1,\ldots,\theta_H\}$ be the registered mechanisms. For
mechanism $\theta$, draw $R_\theta$ independent releases without target
feedback and define
\begin{equation}
\label{eq:release_pvalue}
p_{\theta,r}=\max_{j\in\mathcal J}p_{Q_{\theta,r},j}.
\end{equation}
Before target access, choose local levels
$\alpha_{\mathrm R,\theta}$ with
$\sum_\theta\alpha_{\mathrm R,\theta}\leq\alpha_{\mathrm R}$.
Holm testing within each mechanism identifies
\begin{equation}
\label{eq:holm_count}
V_\theta^{\mathrm{Holm}}
=
\sum_{r=1}^{R_\theta}
\mathbf 1\{\text{Holm rejects }p_{\theta,r}\text{ at }
\alpha_{\mathrm R,\theta}\}
\end{equation}
named releases. Equal allocation $\alpha_{\mathrm R}/H$ is the default.

For the bad-mechanism null
$H^{\mathrm{mech}}_{0,\theta}:\eta_\theta\leq\eta_0$, use
\begin{equation}
\label{eq:mechanism_p}
p_\theta^{\mathrm{mech}}
=
\Pr\{B_\theta\geq V_\theta^{\mathrm{Holm}}\},
\qquad
B_\theta\sim\operatorname{Binomial}(R_\theta,\eta_0),
\end{equation}
followed by Holm correction across mechanisms.

\begin{theorem}[False mechanism-validation control]
\label{thm:mechanism}
Suppose the target sample is independent of all development and release
choices; researchers fix the mechanisms, releases, requirements, and
$\eta_0$ before target access; the component $p$-values are valid; and
releases within each mechanism are independent draws. If the local release
levels sum to at most $\alpha_{\mathrm R}$ and the mechanism family uses Holm
at $\alpha_{\mathrm M}$, then
\[
\Pr\!\left(
\exists\text{ validated }\theta:\eta_\theta\leq\eta_0
\right)
\leq\alpha_{\mathrm R}+\alpha_{\mathrm M}.
\]
The audit may share target records across releases and mechanisms.
\end{theorem}

The recognized count need not be binomial. Outside the inner error event,
$V_\theta^{\mathrm{Holm}}$ is no larger than the latent number of truly valid
releases, which makes \cref{eq:mechanism_p} conservative. Use named-release
mode when the audit must identify particular releases.

\subsection{Direct shared-target mechanism inference}
\label{sec:direct_shared_target}

Use direct shared-target mode when the audit needs only the mechanism claim.
Its logic has three steps. First, score every independent mechanism draw on
the same target.
Second, lower-bound the conditional probability of a favorable score. Third,
subtract the score contribution that invalid releases can make because the
target sample is finite.

Fix slacks $s_j>0$ before target access and define
\begin{equation}
\label{eq:conditional_score}
S_{\theta,r}
=
\prod_{j\in\mathcal J}
\mathbf 1\!\left\{
\overline X_{\theta,r,j}\leq\tau_j-s_j
\right\},
\qquad
\overline X_{\theta,r,j}=\frac1n\sum_{i=1}^nX_{i,\theta,r,j}.
\end{equation}
Every release uses the same target records $Z_{1:n}$. Conditional on those
records, the scores are independent Bernoulli observations because the
releases are independent draws from $G_\theta$.

The direct correction needs a ceiling $\kappa_\theta$ on the worst-case
probability that an invalid release crosses every stricter empirical cutoff.
Suppose $X_j\in[a_j,b_j]$ and
$a_j<\tau_j-s_j<\tau_j<b_j$. Let $d_{\mathrm B}$ be Bernoulli KL divergence.
If a release is invalid, at least one true mean is at or above its limit.
Hoeffding's bounded-loss Chernoff inequality \citep{hoeffding1963probability}
then supplies the explicit ceiling
\begin{equation}
\label{eq:invalid_score_ceiling}
\sup_{Q:A(Q)=0}\mathbb E_Z[S_\theta(Q,Z)]
\leq
\kappa_\theta
:=
\max_j
\exp\!\left[
-n\,d_{\mathrm B}\!\left(
\frac{\tau_j-s_j-a_j}{b_j-a_j}
\middle\Vert
\frac{\tau_j-a_j}{b_j-a_j}
\right)
\right].
\end{equation}

\paragraph{Interpretation.}
Direct mode observes favorable scores, not valid releases. A favorable score
can come from a valid release or from an invalid release that passes by chance
on the finite target. The following decomposition separates these two sources.
The method lower-bounds total favorable-score mass and subtracts a
high-probability upper bound on invalid false-pass mass.

Formally,
\begin{equation}
\label{eq:score_decomposition}
\underbrace{\pi_\theta(Z)}_{\text{all favorable scores}}
=
\mathbb E_Q[A(Q)S_\theta(Q,Z)\mid Z]
+
\underbrace{\beta_\theta(Z)}_{\text{invalid false-pass mass}}
\leq
\eta_\theta+\beta_\theta(Z),
\end{equation}
where
$\beta_\theta(Z)=\mathbb E_Q[(1-A(Q))S_\theta(Q,Z)\mid Z]$.
Equation~\eqref{eq:invalid_score_ceiling} implies
$\mathbb E_Z[\beta_\theta(Z)]\leq\kappa_\theta$, so Markov's inequality gives
$\beta_\theta(Z)\leq\kappa_\theta/\alpha_{\mathrm Z,\theta}$ outside a
target-side event of probability $\alpha_{\mathrm Z,\theta}$.

Let $\underline\pi_\theta^{\mathrm{CP}}$ be the exact one-sided binomial
lower bound on $\pi_\theta(Z)$ at local level
$\alpha_{\mathrm Q,\theta}$. Rearranging
\cref{eq:score_decomposition} yields
\begin{equation}
\label{eq:direct_reliability_bound}
L_\theta^{\mathrm{dir}}
=
\left[
\underline\pi_\theta^{\mathrm{CP}}
-
\frac{\kappa_\theta}{\alpha_{\mathrm Z,\theta}}
\right]_+.
\end{equation}

Thus, \cref{eq:direct_reliability_bound} is a score-rate lower bound minus a
target-side false-pass allowance.

The certificate separates two uncertainty sources. The level
$\alpha_{\mathrm Q,\theta}$ controls finite mechanism-draw uncertainty through
the binomial lower bound, while $\alpha_{\mathrm Z,\theta}$ controls
false-pass contamination from the shared finite target through the subtraction
term. The theorem allocates these budgets separately and combines their
failure probabilities with a union bound.

\begin{theorem}[Direct reliability inference with a shared target]
\label{thm:conditional_shared_target}
Suppose Assumption~\ref{ass:audit} holds and
$Q_{\theta,1},\ldots,Q_{\theta,R_\theta}\overset{\mathrm{iid}}{\sim}G_\theta$
within each mechanism. Assume also that the registered losses have known bounds. If
$\sum_\theta\alpha_{\mathrm Q,\theta}\leq\alpha_{\mathrm Q}$ and
$\sum_\theta\alpha_{\mathrm Z,\theta}\leq\alpha_{\mathrm Z}$, then
\[
\Pr\!\left(
\eta_\theta\geq L_\theta^{\mathrm{dir}}
\text{ for every }\theta
\right)
\geq1-\alpha_{\mathrm Q}-\alpha_{\mathrm Z}.
\]
The design may share the same target records across all releases and
mechanisms. The result does not assume independence or PRDS among
release-level $p$-values.
\end{theorem}

The proof is in \appref{app:proofs}.

\paragraph{Why the contamination penalty is necessary.}
The Markov step may appear conservative, but no smaller dependence on
$\alpha_{\mathrm Z,\theta}$ follows from the mean constraint alone.

\begin{proposition}[Optimality of expectation-only contamination correction]
\label{prop:markov_sharpness}
Let $\mathcal B_\kappa$ contain every random variable $B\in[0,1]$ with
$\mathbb E B\leq\kappa$. For $c\in(0,1]$,
\[
\sup_{B\in\mathcal B_\kappa}\Pr(B\geq c)
=\min\{1,\kappa/c\}.
\]
Consequently, a uniform $(1-\alpha)$ ceiling based only on
$\mathbb E B\leq\kappa$ cannot be smaller than
$\min\{1,\kappa/\alpha\}$, apart from the convention at equality.
\end{proposition}

For $c\leq\kappa$, the constant variable $B=c$ attains probability one. For
$c>\kappa$, placing mass $\kappa/c$ at $c$ and the remaining mass at zero
attains equality. The proof is in \appref{app:proofs}. Direct mode is
informative only when the registered score and target size make its
contamination allowance small.

\paragraph{Smooth structured certificate.}
A registered ramp score supplies structure absent from the mean-only result.
With invalid-release mean ceiling $\kappa_h$, per-record target sensitivity
$\rho_h$, and conditional-mean lower bound $\underline\pi_h$,
Theorem~\ref{thm:smooth_target_concentration} gives
\begin{equation}
\label{eq:smooth_direct_main}
\boxed{
L^{\mathrm{smooth}}
=
\left[
\underline\pi_h-
\min\!\left\{1,\kappa_h+
\rho_h\sqrt{\frac n2\log\frac1{\alpha_{\mathrm Z}}}
\right\}
\right]_+ .}
\end{equation}
Bounded sensitivity thus replaces the multiplicative Markov allowance by an
additive radius; \appref{app:continuous_scores} gives the construction and proof,
and \appref{app:smooth_target_confirmation} gives its confirmation.

\paragraph{Registered stratified targets.}
For fixed stratum sizes and registered weights, Proposition~\ref{prop:stratified_target}
validates both modes, including without-replacement sampling. A balanced
reserve targets an equal-class mixture unless registered weights restore
another prevalence; details are in \appref{app:stratified_targets}.

\subsection{Evidence of unreliability}
\label{sec:unreliability}

The direct certificate in \cref{sec:direct_shared_target} is a one-sided lower
reliability bound. To establish that a mechanism is unreliable, we instead
use corrected named-release violation tests. The opposite-direction null is
$H^{\mathrm{viol}}_{0,Q,j}:\mu_{Q,j}\leq\tau_j$. Reflecting the empirical
Bernstein bound gives
\[
L_{Q,j}(\delta)
=
\bar X_{Q,j}
-
\sqrt{\frac{2\widehat V_{Q,j}\log(2/\delta)}{n}}
-
\frac{7(b_j-a_j)\log(2/\delta)}{3(n-1)},
\qquad
q_{Q,j}=\inf\{\delta:L_{Q,j}(\delta)>\tau_j\}.
\]
Apply Holm at $\beta_{\mathrm R}$ to the registered
release--requirement family. Let $W_\theta$ count releases with at least one
detected violation. Outside the inner error event,
$R_\theta-W_\theta$ upper-bounds the latent number of valid releases. The
mechanism value is
\[
p_\theta^{\mathrm{viol}}
=
\Pr\{B_\theta\leq R_\theta-W_\theta\},
\qquad
B_\theta\sim\operatorname{Binomial}(R_\theta,\eta_0),
\]
followed by Holm at $\beta_{\mathrm M}$.

\begin{theorem}[False mechanism-violation control]
\label{thm:mechanism_violation}
Suppose the target and release conditions of \cref{thm:mechanism} hold and
each $q_{Q,j}$ is super-uniform under its violation null. Then
\[
\Pr\!\left(
\exists\text{ mechanism declared unreliable }\theta:
\eta_\theta\geq\eta_0
\right)
\leq\beta_{\mathrm R}+\beta_{\mathrm M}.
\]
\end{theorem}

Lower and upper reliability bounds use separate directional budgets. If each
direction has total error 0.05, displaying them together gives at least 90\%
joint coverage by a union bound, not a conventional 95\% two-sided interval.

\subsection{Planning and registered mode choice}
\label{sec:planning_mode}

Before target access, researchers fix each family budget and its local
allocations. \Cref{tab:error_budgets} in \appref{app:continuous_scores} maps
the budgets to their error events.

\noindent\fbox{\begin{minipage}{0.96\linewidth}\footnotesize
\textbf{Mode-separation consequence.}
Named identification pays a $\log R$ target-sample term; direct target-side
terms do not grow with $R$ (Theorem~\ref{thm:mode_separation}). Use named mode
for release identities or strong per-release evidence, and direct mode for
aggregate mechanism evidence. Direct is infeasible if contamination exhausts
$1-\eta_0$; expensive draws favor named mode or an unresolved decision.
\end{minipage}}

Researchers choose the mode before target access. Proposition~\ref{prop:hybrid}
permits a preregistered hybrid with split budgets; an uncorrected maximum would
use target outcomes to choose the mode and invalidate the guarantee.
The registry likewise fixes hard versus smooth scoring and all $h_j$;
post-target maximization requires a preregistered simultaneous allocation, or
the stated coverage need not hold.

Both modes report one-sided, not two-sided, reliability bounds.
\appref{app:additional_procedures} covers planning, candidate streams, and
exposure claims; \appref{app:sample_size} gives a sufficient audit-size bound.

%% file: sections/experiments.tex
\section{Experiments}
\label{sec:experiments}

We froze each registry before its confirmatory run. Appendices
\ref{app:mechanism_experiments}, \ref{app:claim_status}, and
\ref{app:reproducibility} describe the study designs, claim status, and
reproduction procedures. Unless stated otherwise, total error is 0.05. We select learners and
thresholds on development data and hold them fixed for the target audit.

\subsection{False validation after search}
\label{sec:false_validation}

\noindent\textbf{Candidate search.}
Across 5{,}000 trials in the 20-candidate, three-requirement study, point
thresholds validate an invalid candidate every time; separate tests give
44.0\% familywise false validation, and Holm gives 1.94\%. Selecting the
smallest-$p$ candidate and reusing its target gives 43.46\%, compared with
3.08\% on a sealed target (\appref{app:release_results}).

\noindent\textbf{Mechanism calibration.}
At $R=50$, audit size 500, and $\eta_0=0.8$, the observed passing fraction
falsely validates 99.70\% of boundary mechanisms. An outer binomial test gives
5.9\%; named inference gives 0.56\% and 99.57\% power at reliability 0.98.
\Cref{fig:mechanism_control} and the rare-group study appear in Appendices
\ref{app:mechanism_experiments} and \ref{app:heldout_regrets}.

\subsection{Direct shared-target versus named-release mode}
\label{sec:mode_comparison}

The primary comparison uses three correlated bounded requirements, continuous
margins, and one target fluctuation shared by every release. Moderate evidence
makes Holm rarely certify an individual release, so direct mode can benefit;
the high-signal design lowers variance and makes named evidence decisive.

We tuned both modes on the same $R=n=1{,}000$ pilot with registered grids and
the same mean-power objective, then froze them before a new-seed confirmation.
\appref{app:mechanism_experiments} gives the grids and deterministic tie rule.

\begin{table}[H]
\centering
\small
\setlength{\tabcolsep}{6pt}
\input{tables/direct_multirequirement}
\caption{Validation rates (\%) over 500 confirmatory repetitions with
$R=n=1{,}000$ and $\eta_0=0.8$. The boundary column measures false
validation; the other columns measure power.}
\label{tab:direct_multirequirement}
\end{table}

\Cref{tab:direct_multirequirement} shows that direct mode raises power from
4.2\% to 56.4\% at reliability 0.90 and from
5.6\% to 64.2\% at 0.95. No direct false validation occurs in 500 boundary
trials (one-sided 95\% Monte Carlo upper bound 0.60\%).

A boundary diagnostic shows why direct mode must subtract the term in
\cref{eq:direct_reliability_bound}: omitting it gives 8.05\% false
validation, while the corrected certificate always abstains
(\appref{app:false_pass_diagnostic}). In a separate one-requirement,
high-reliability stress test, smooth concentration has 86.45\% power at
$\eta=0.99$ versus zero for both Markov certificates, with 0/2{,}000 boundary
errors; all three remain unresolved at $\eta=0.95$
(\appref{app:smooth_target_confirmation}).

\begin{figure}[H]
\centering
\includegraphics[width=0.62\textwidth]{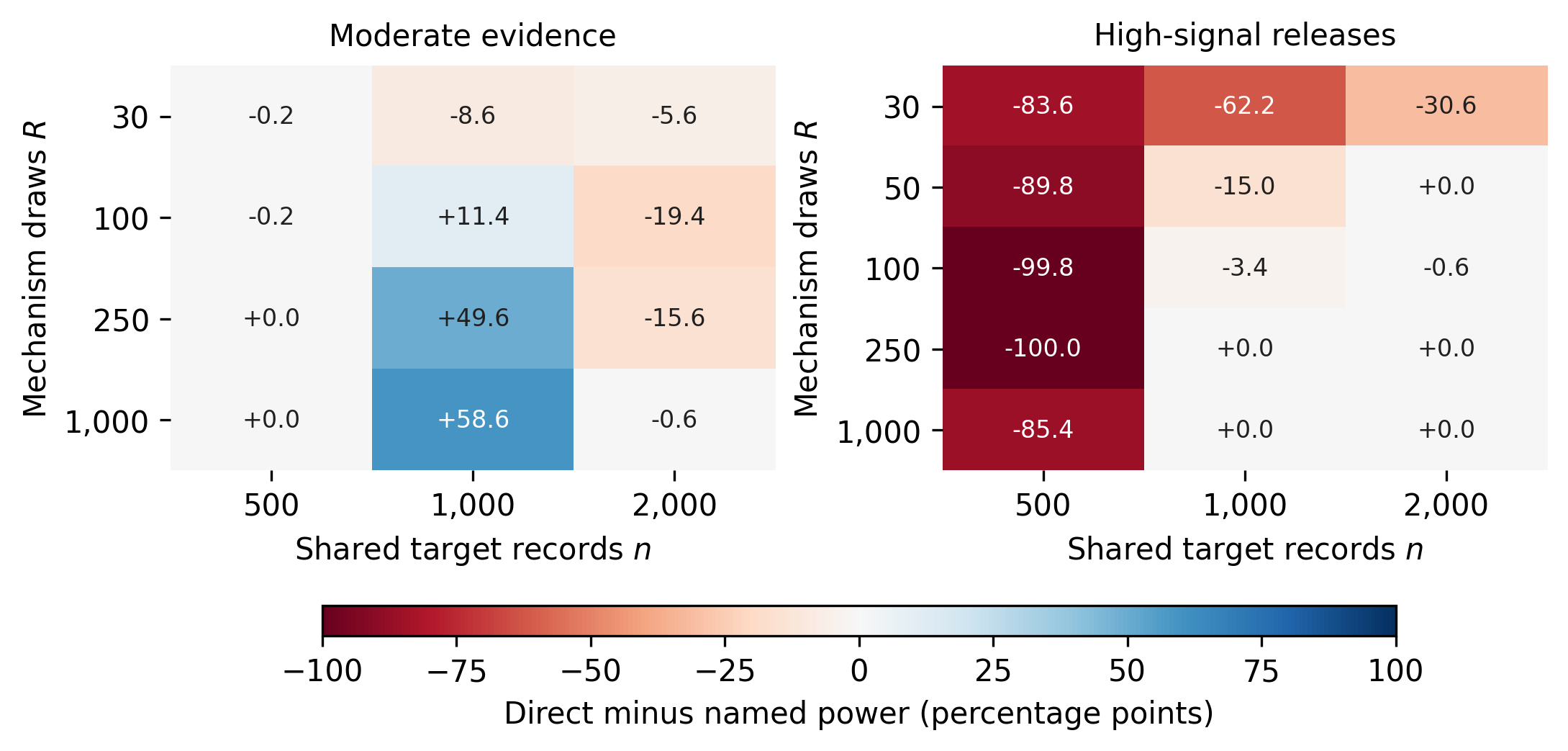}
\caption{Direct minus named power at reliability 0.95 after we planned both
methods on the same pilot. Positive cells favor direct shared-target mode;
negative cells favor named-release mode.}
\label{fig:direct_multirequirement_power}
\end{figure}

\Cref{fig:direct_multirequirement_power} shows a 49.6-point gain at
$R=250,n=1{,}000$, vacuity at $n=500$, and nine ties versus 36 losses in the
high-signal design. Planning details are in \appref{app:mechanism_experiments}.

\subsection{Sealed mechanism studies}
\label{sec:sealed_studies}

Each lightweight-neural draw refits the generator and proxy classifier. Fixed
balanced reserves target an equal class mixture under
Proposition~\ref{prop:stratified_target}, not natural-prevalence risk. The
corrected Covertype moderate bounds are 0.948 named and 0.885 direct; Online
Shoppers remains unresolved \citep{blackard1998covertype,sakar2018onlineshoppers}.
These corrections establish sampling validity, not direct-mode superiority.

\begin{table}[H]
\centering
\scriptsize
\setlength{\tabcolsep}{3.5pt}
\input{tables/neural_outcomes_compact}\caption{All nine corrected lightweight-neural configurations. Parentheses
give named/direct lower bounds; $U$ is the violation-side upper bound.}
\label{tab:neural_outcomes_compact}
\end{table}

As summarized in \cref{tab:neural_outcomes_compact}, prospective audits cover
full-pipeline Rice--TVAE and a frozen 20 Newsgroups
unigram \citep{cinar2019rice,xu2019modeling,lang1995newsweeder}. Text
high and moderate validate under both modes (named/direct: 0.981/0.974 and
0.991/0.987); degraded is unresolved (0.157/0.070). The mechanism is not a
modern language model, and no real audit yields a direct-only decision
(Appendices~\ref{app:magic_sealed}--\ref{app:text_direct}).

%% file: tables/direct_multirequirement.tex
\begin{tabular}{lrrr}
\toprule
Method &
$\eta=0.80$ &
$\eta=0.90$ &
$\eta=0.95$ \\
\midrule
Named-release Holm & 0.2 & 4.2 & 5.6 \\
Direct shared-target & 0.0 & 56.4 & 64.2 \\
Preregistered hybrid & 0.0 & 2.4 & 57.2 \\
\bottomrule
\end{tabular}

%% file: tables/neural_outcomes_compact.tex
\begin{tabularx}{\textwidth}{@{}lXXX@{}}
\toprule
Dataset & High fidelity & Moderate evidence & Degraded \\
\midrule
CDC & both validate (.938/.917) & both validate (.948/.931) & violation ($U=.280$) \\
Covertype & both validate (.938/.910) & both validate (.948/.885) & violation ($U=.062$) \\
Online Shoppers & unresolved (.449/.363) & unresolved (.002/.017) & violation ($U=.703$) \\
\bottomrule
\end{tabularx}

%% file: sections/limitations.tex
\section{Limits of the claim}
\label{sec:limitations}

\method{} guarantees one registered specification on an independent target;
it does not cover unregistered subgroups, nor does relative transfer imply
absolute adequacy. Balanced claims concern equal class mixtures absent
prevalence-restoring weights. Decisions depend on registered limits, slacks,
reliability target, and error split; we do not test every policy.

The mean-only subtraction can be vacuous but is uniformly sharp
(Proposition~\ref{prop:markov_sharpness}); direct mode also trades target
labels for mechanism draws (\appref{app:additional_procedures}). Smooth
concentration is not uniformly tighter: its positive confirmation is
one-requirement and high-reliability, whereas its multi-requirement allowance
is vacuous in the text audit. Ramp selection and sharper sensitivity bounds
remain open. The text mechanism is a simple unigram, three neural analyses are
corrective, and no real audit yields a direct-only decision.

%% file: sections/conclusion.tex
\section{Conclusion}
\label{sec:conclusion}

\method{} separates release and mechanism certification. Named mode preserves
identities; direct mode subtracts a high-probability upper bound on invalid
releases' favorable-score mass. Under the registered losses, target, and frozen
mechanism law, it lower-bounds a future independent release's validity
probability.

%% file: sections/appendix.tex
\section{Extended related work}
\label{app:related_work}

The appendix follows the audit from theory to evidence.
Appendices~\ref{app:related_work}--\ref{app:proofs} develop related work,
alternative certificates, and proofs; Appendices~\ref{app:sample_size}--
\ref{app:mechanism_experiments} cover sample size, planning, calibration, and
mechanism simulations. Appendices~\ref{app:claim_status}--
\ref{app:standin_audits} report claim status and the real-data audits, while
Appendices~\ref{app:learning_procedure}--\ref{app:reproducibility} document
learning procedures, generators, attacks, transformation screening, and
reproducibility. Readers can follow this order or jump to a cited appendix.
Principal alternatives appear in \appref{app:partial_conjunction},
\appref{app:block_witness}, and \appref{app:component_tests}.

\paragraph{Evaluation and risk control.}
Synthetic tables can preserve common prediction metrics while distorting
minority regions or statistical inference
\citep{vanbreugel2023realerrors,decruyenaere2024inferential}.
Exchangeability tests ask a broader distributional question
\citep{lofstrom2024exchangeability}; our claims concern only the registered
losses. Learn-then-Test and its adaptive variants provide the closest
risk-control foundation \citep{angelopoulos2021ltt,zecchin2025adaptive}.
Empirical Bernstein bounds and Holm's procedure supply our inner tests
\citep{maurer2009empirical,holm1979simple}.

\paragraph{Partial conjunction and reliability.}
A partial-conjunction test asks whether at least $k$ hypotheses in a family
are false \citep{benjamini2008screening}. This matches a mechanism audit that
needs a valid-release count but not release identities. Tail-Simes and
tail-Fisher can be conservative in moderate-signal settings; conditional
tests trade stronger assumptions for power \citep{liang2025conditional}.
Binomial reliability demonstration and acceptance sampling ask whether
repeated units meet a reliability target
\citep{kapur1987optimal,markowski2002misclassification}. These models
typically treat inspection error as a fixed property of a unit-level test.
Hierarchical generalized linear mixed models instead introduce a specified
random-effects distribution for model-based inference
\citep{breslow1993approximate}; they do not supply the simultaneous
finite-sample shared-target certificate under our bounded-loss and
registered-sampling assumptions.
Here, release validity is latent and every inspection reuses one random target
sample. Target reuse couples false-pass errors across releases, and their
conditional rate is unknown. The direct certificate conditions on the shared
target to recover independence over mechanism draws, then controls the random
false-pass contribution on a separate target-side event. This two-axis
construction is the part not supplied by classical reliability
demonstration.

\paragraph{Generators and attacks.}
AIM is a workload-adaptive differentially private mechanism
\citep{mckenna2022aim}; TabDDPM and TabSyn are learned tabular generators
\citep{kotelnikov2023tabddpm,zhang2024tabsyn}. Differential privacy is a
mechanism guarantee \citep{dwork2006calibrating}, whereas DOMIAS, Gen-LRA, and
MIA-EPT measure specific empirical attacks
\citep{vanbreugel2023domias,ward2025genlra,german2025miaept}. A bound for a
fixed attack suite remains conditional on that suite.

\section{Independent-batch partial conjunction}
\label{app:partial_conjunction}

Partial conjunction is a secondary mechanism-only mode when release-level
$p$-values are independent or satisfy a justified PRDS condition. It does not
justify dependence created by reusing one target sample. Sort the release
values as
$p_{\theta,(1)}\leq\cdots\leq p_{\theta,(R_\theta)}$ and define
\begin{equation}
\label{eq:pc_p}
\widehat p^{\mathrm{PC}}_{\theta,k}
=
\min\!\left(
 \min_{j=k}^{R_\theta}
 \frac{R_\theta-k+1}{j-k+1}p_{\theta,(j)},1
\right),
\qquad
\widetilde p^{\mathrm{PC}}_{\theta,k}
=
\max_{\ell\leq k}\widehat p^{\mathrm{PC}}_{\theta,\ell}.
\end{equation}
The tail-Simes value tests
$H_{\theta,k}:K_\theta<k$, where
$K_\theta=\sum_r A(Q_{\theta,r})$. The simultaneous lower count is
\begin{equation}
\label{eq:collective_count}
\widehat K^{\mathrm{LB}}_\theta
=
\max\left\{
k:\widetilde p^{\mathrm{PC}}_{\theta,k}
\leq\alpha_{\mathrm R,\theta}
\right\}.
\end{equation}
Replacing $V_\theta^{\mathrm{Holm}}$ by this count in
\cref{eq:mechanism_p} gives $p_\theta^{\mathrm{mech,coll}}$.

\begin{theorem}[Independent-batch collective validation]
\label{thm:mechanism_collective}
Under the assumptions of \cref{thm:mechanism}, condition on the sampled
releases and fixed development choices. Suppose the true-null release
$p$-values are conditionally super-uniform and independent or PRDS within
each mechanism. If the local levels sum to at most
$\alpha_{\mathrm R}$, then Holm across mechanisms at
$\alpha_{\mathrm M}$ gives
\[
\Pr\!\left(
\exists\text{ validated }\theta:\eta_\theta\leq\eta_0
\right)
\leq
\alpha_{\mathrm R}+\alpha_{\mathrm M}.
\]
\end{theorem}

\begin{proposition}[Collective-count dominance]
\label{prop:collective_dominance}
At the same local level,
$\widehat K_\theta^{\mathrm{LB}}\geq
V_\theta^{\mathrm{Holm}}$ for every mechanism in every run.
\end{proposition}

\section{Stratified target designs}
\label{app:stratified_targets}

For fixed stratum sizes $n_g$ and registered target weights $w_g$, use
$\overline X_{\theta,r,j}=\sum_gw_g\overline X_{\theta,r,j,g}$. The
invalid-release ceiling becomes
\begin{equation}
\label{eq:stratified_ceiling}
\kappa_{\theta}^{\mathrm{strat}}
=
\max_j\exp\!\left\{
-\frac{2s_j^2}{(b_j-a_j)^2\sum_g w_g^2/n_g}
\right\}.
\end{equation}
When $w_g/n_g$ is constant across strata, the sharper Bernoulli--KL ceiling
in \cref{eq:invalid_score_ceiling} remains valid with $n=\sum_gn_g$.

\begin{proposition}[Stratified shared-target extension]
\label{prop:stratified_target}
Under a registered fixed-size stratified design, the target sample remains
independent of all development and release choices. Each stratum supplies
either i.i.d.\ records or a simple random sample without replacement. Weighted
Hoeffding inversion gives valid named-release component $p$-values.
Replacing the empirical means and $\kappa_\theta$ by their weighted forms in
\cref{eq:stratified_ceiling} preserves the simultaneous conclusion of
\cref{thm:conditional_shared_target}.
\end{proposition}

\section{Continuous shared-target scores}
\label{app:continuous_scores}

The audit may replace the hard factor in \cref{eq:conditional_score} with a
registered score in $[0,1]$. Let $c_j=\tau_j-s_j$ and fix
$0<h_j\leq c_j-a_j$. Define
\[
\phi_{j,h}(x)
=
\begin{cases}
1, & x\leq c_j-h_j,\\
(c_j-x)/h_j, & c_j-h_j<x<c_j,\\
0, & x\geq c_j,
\end{cases}
\qquad
S_h(Q,Z)=\prod_j\phi_{j,h}(\overline X_{Q,j}).
\]
Conditional release scores are bounded rather than Bernoulli, so a
Bernoulli--KL inversion lower-bounds their conditional mean. The ramp also
supplies target-side structure that the hard score lacks.

Write $B_j=b_j-a_j$, $u_j(t)=(t-a_j)/B_j$, and
$v_j=(\tau_j-a_j)/B_j$. Set
\begin{align}
\kappa_h
&=
\max_j\frac1{h_j}\int_{c_j-h_j}^{c_j}
\exp\{-n d_{\mathrm B}(u_j(t)\Vert v_j)\}\,dt,
\label{eq:smooth_kappa}\\
\rho_h
&=
\min\!\left\{1,\frac1n\sum_j\frac{B_j}{h_j}\right\},
\qquad
r_h(\alpha_{\mathrm Z})
=
\rho_h\sqrt{\frac n2\log\frac1{\alpha_{\mathrm Z}}}.
\label{eq:smooth_radius}
\end{align}

\begin{theorem}[Smooth shared-target concentration]
\label{thm:smooth_target_concentration}
Suppose the target records are iid and jointly independent of the iid
mechanism draws. Let $\underline\pi_h$ be a one-sided
$(1-\alpha_{\mathrm Q})$ lower confidence bound for the conditional mean of
$S_h$. Then
\[
L^{\mathrm{smooth}}
=
\left[
\underline\pi_h-
\min\{1,\kappa_h+r_h(\alpha_{\mathrm Z})\}
\right]_+
\]
satisfies $\Pr\{\eta\geq L^{\mathrm{smooth}}\}
\geq1-\alpha_{\mathrm Q}-\alpha_{\mathrm Z}$. The statement holds
simultaneously across mechanisms when their local error levels sum to the
declared family budgets.
\end{theorem}

\begin{proof}
For any $x$, the ramp identity
$\phi_{j,h}(x)=h_j^{-1}\int_{c_j-h_j}^{c_j}\mathbf 1\{x\leq t\}\,dt$
and the bounded-loss Chernoff inequality give \cref{eq:smooth_kappa} for an
invalid release. Replacing one target record changes $S_h$ by at most
$\rho_h$, uniformly over releases. The same bound applies after averaging
over invalid mechanism draws. McDiarmid's inequality
\citep{mcdiarmid1989bounded} therefore bounds their target-conditional score
mass by $\kappa_h+r_h$. Combining that event with the conditional bounded-mean
lower bound proves the theorem. The expectation-only Markov allowance for the
same ramp is $\min\{1,\kappa_h/\alpha_{\mathrm Z}\}$; smooth concentration is
strictly tighter whenever
$r_h<\kappa_h(1/\alpha_{\mathrm Z}-1)$.
\end{proof}
The artifact evaluates \cref{eq:smooth_kappa} with a conservative right
Riemann sum.

\paragraph{Ramp-width selection.}
The widths $h_j$ are development-time design parameters. An auditor may
compare a finite grid using only development or pilot estimates of the
conditional score rate, $\kappa_h$, and $\rho_h$, then freeze one configuration
before target access. Wider ramps reduce worst-case target sensitivity but may
lower favorable-score rates; narrower ramps approach the hard score but can
produce a large concentration allowance. We do not claim an optimal
ramp-selection rule.

\Cref{tab:error_budgets} summarizes the distinct directional error events.

\begin{table}[H]
\centering
\scriptsize
\setlength{\tabcolsep}{4pt}
\begin{tabularx}{\textwidth}{@{}llX@{}}
\toprule
Mode or direction & Budget & Error event controlled \\
\midrule
Named lower bound & $\alpha_{\mathrm R}$ & false recognition among named releases \\
Named lower bound & $\alpha_{\mathrm M}$ & false validation among mechanisms \\
Direct lower bound & $\alpha_{\mathrm Q}$ & conditional score-probability coverage \\
Direct lower bound & $\alpha_{\mathrm Z}$ & shared-target false-pass contamination \\
Violation upper bound & $\beta_{\mathrm R}$ & false release-level violation \\
Violation upper bound & $\beta_{\mathrm M}$ & false mechanism-level violation \\
Registered hybrid & $\alpha_{\mathrm N},\alpha_{\mathrm D}$ & named and direct component coverage \\
\bottomrule
\end{tabularx}
\caption{Error budgets and the events they control.}
\label{tab:error_budgets}
\end{table}

\section{A two-axis block certificate}
\label{app:block_witness}

The expectation-only lower bound in Proposition~\ref{prop:markov_sharpness} is sharp.
One way to obtain more structure is to freeze a random partition of the target
into $B=\lfloor n/m\rfloor$ disjoint blocks of size $m$. Let $W_{rb}\in[0,1]$
be the registered hard or ramped score of release $r$ on block $b$, and let
$\overline W=(RB)^{-1}\sum_{r,b}W_{rb}$. Define $\kappa_m$ as in
\cref{eq:invalid_score_ceiling}, with $m$ in place of $n$, and
\[
r_{R,B}(\alpha)
=
\sqrt{\frac12\left(\frac1R+\frac1B\right)\log\frac1\alpha}.
\]

\begin{theorem}[Block-witness reliability certificate]
\label{thm:block_witness}
Suppose the target records are iid and jointly independent of the iid
mechanism draws. For one mechanism,
\[
L^{\mathrm{blk}}
=
\left[
\frac{\overline W-r_{R,B}(\alpha)-\kappa_m}{1-\kappa_m}
\right]_+
\]
satisfies $\Pr\{\eta\geq L^{\mathrm{blk}}\}\geq1-\alpha$. For several
mechanisms, the same statement holds simultaneously when their local error
levels sum to at most $\alpha$.
\end{theorem}

To see this, view $\overline W$ as a function of the $R$ releases and $B$
target blocks. Replacing one release changes it by at most $1/R$; replacing
one block changes it by at most $1/B$. McDiarmid's inequality gives
$\mathbb E\overline W\geq\overline W-r_{R,B}(\alpha)$ outside an event of
probability $\alpha$. An invalid release has expected witness at most
$\kappa_m$, while a valid release contributes at most one, so
$\mathbb E\overline W\leq\kappa_m+(1-\kappa_m)\eta$. Rearrangement proves
the result. The block size trades a smaller $\kappa_m$ against fewer
independent blocks, so researchers must choose it during development. We
implement this certificate in the artifact but do not select it after target
access. We treat it as a structured theoretical extension. The current
experiments do not establish when it improves on the expectation-only
certificate.

\section{Proofs}
\label{app:proofs}

\paragraph{Proof of \cref{thm:fwer}.}
If $Q$ is invalid, some component null $H_{0,Q,j^\star}$ is true. For any
$t\in[0,1]$,
\[
\Pr(p_Q\leq t)
=
\Pr\!\left(\max_j p_{Q,j}\leq t\right)
\leq
\Pr(p_{Q,j^\star}\leq t)
\leq t.
\]
Thus $p_Q$ is valid for the proxy-level null. Holm controls familywise error
over the proxy-level nulls under arbitrary dependence
\citep{holm1979simple}.

\paragraph{Proof of \cref{thm:mechanism}.}
Let $K_\theta=\sum_{r=1}^{R_\theta}A(Q_{\theta,r})$ be the unobserved number
of truly valid releases. Holm controls false release identification within
mechanism $\theta$ at $\alpha_{\mathrm R,\theta}$. A union bound over the
registered mechanisms therefore gives
\[
\Pr\!\left(
V_{\theta}^{\mathrm{Holm}}\leq K_\theta
\text{ for every }\theta
\right)
\geq
1-\sum_\theta\alpha_{\mathrm R,\theta}
\geq 1-\alpha_{\mathrm R}.
\]
Under
$H^{\mathrm{mech}}_{0,\theta}$, the binomial upper-tail value computed from
$K_\theta$ is super-uniform. The tail decreases with its observed count, so
the value computed from $V_{\theta}^{\mathrm{Holm}}$ is no smaller than that
from $K_\theta$. Holm is monotone in its input $p$-values. Conditional on no
inner error, if the test rejects using $V_{\theta}^{\mathrm{Holm}}$, it also
rejects when using $K_\theta$. Therefore, a false mechanism rejection is
contained in the rejection event of an oracle Holm test based on $K_\theta$, whose probability
is at most $\alpha_{\mathrm M}$. A union bound with the inner error event
completes the proof.

\paragraph{Proof of Proposition~\ref{prop:stratified_target}.}
For requirement $j$, write
$\widehat\mu_j=\sum_gw_g\overline X_{j,g}$. Each observation in stratum
$g$ has coefficient $w_g/n_g$, so weighted Hoeffding gives, for $t>0$,
\[
\Pr\{\widehat\mu_j-\mu_j\leq-t\}
\leq
\exp\!\left\{
-\frac{2t^2}{(b_j-a_j)^2\sum_gw_g^2/n_g}
\right\}.
\]
The same inequality holds for simple random sampling without replacement
within a finite stratum because the without-replacement sum is no less
concentrated than its with-replacement counterpart
\citep{hoeffding1963probability}. Inverting the corresponding upper-tail
bound gives a super-uniform named-release component $p$-value.

If $A(Q)=0$, some requirement has $\mu_j\geq\tau_j$. A favorable score then
requires $\widehat\mu_j-\mu_j\leq-s_j$, and the preceding display yields
\cref{eq:stratified_ceiling}. When $w_g/n_g=1/n$ for every stratum, the
weighted sum is the ordinary average of $n=\sum_gn_g$ bounded observations.
The bounded Bernoulli--KL Chernoff argument applies to independent,
not-necessarily-identically-distributed observations with that average mean;
the without-replacement comparison gives the same conclusion for finite
strata. The remainder of the proof is identical to
\cref{thm:conditional_shared_target}: condition on the shared stratified
sample for the release-score bound, control its invalid false-pass mass on a
target-side event, and union-bound the registered mechanisms.

\paragraph{Proof of \cref{thm:conditional_shared_target}.}
Write $A_\theta(Q)$ for the indicator that release $Q$ satisfies every
registered requirement. If $A_\theta(Q)=0$, some requirement $j^\star$ has
mean at least $\tau_{j^\star}$. Since $S_\theta$ is no larger than the
indicator for the $j^\star$ sample-mean event, the bounded-loss Chernoff
inequality gives
\[
\sup_{Q:A_\theta(Q)=0}\mathbb E_Z[S_\theta(Q,Z)]
\leq \kappa_\theta.
\]
Define the target-conditional contribution of invalid releases by
\[
\beta_\theta(Z)
=
\mathbb E_Q[
(1-A_\theta(Q))S_\theta(Q,Z)\mid Z
].
\]
Fubini's theorem and the preceding bound yield
$\mathbb E_Z[\beta_\theta(Z)]\leq\kappa_\theta$. Markov's inequality
therefore gives
\[
\Pr_Z\!\left\{
\beta_\theta(Z)>
\kappa_\theta/\alpha_{\mathrm Z,\theta}
\right\}
\leq\alpha_{\mathrm Z,\theta}.
\]

Conditional on $Z$, the registered scores from independent mechanism draws
are iid Bernoulli with mean
$\pi_\theta(Z)=\mathbb E_Q[S_\theta(Q,Z)\mid Z]$. The exact one-sided
binomial bound satisfies
$\underline\pi^{\mathrm{CP}}_\theta\leq\pi_\theta(Z)$ except on an event of conditional
probability at most $\alpha_{\mathrm Q,\theta}$. On the intersection of
these two events,
\[
\pi_\theta(Z)
=
\mathbb E_Q[A_\theta(Q)S_\theta(Q,Z)\mid Z]
+\beta_\theta(Z)
\leq
\eta_\theta+\frac{\kappa_\theta}{\alpha_{\mathrm Z,\theta}}.
\]
Rearranging gives \cref{eq:direct_reliability_bound}. Union bounds over the
registered mechanisms and the two error families give the stated coverage.
The argument does not use the dependence among release-level $p$-values.

\paragraph{Proof of Proposition~\ref{prop:markov_sharpness}.}
Markov's inequality gives
$\Pr(B\geq c)\leq\min\{1,\kappa/c\}$. If $c\leq\kappa$, the constant
variable $B=c$ attains one. If $c>\kappa$, the variable that equals $c$ with
probability $\kappa/c$ and zero otherwise attains $\kappa/c$. Hence the bound
is exact. If a proposed uniform ceiling is smaller than
$\kappa/\alpha$, choose a threshold just above that ceiling and the same
two-point construction makes its exceedance probability larger than
$\alpha$. Clipping at one handles $\kappa/\alpha>1$.

\paragraph{Proof of Proposition~\ref{prop:direct_planning}.}
The invalid-release ceiling is
$\kappa_\theta=\exp(-nd_\theta)$. Even if every release scores one, the
direct lower bound cannot exceed
$1-\kappa_\theta/\alpha_{\mathrm Z,\theta}$. Exceeding $\eta_0$ therefore
requires
$\exp(-nd_\theta)<\alpha_{\mathrm Z,\theta}(1-\eta_0)$, which gives the
target-size condition.

At fixed $n$, suppose this condition holds and write
$c_\theta=\kappa_\theta/\alpha_{\mathrm Z,\theta}$. With all
$R_\theta$ scores equal to one, the exact lower binomial bound is
$\alpha_{\mathrm Q,\theta}^{1/R_\theta}$. No other score outcome can give a
larger lower bound. Direct validation therefore requires
$\alpha_{\mathrm Q,\theta}^{1/R_\theta}>
\eta_0+c_\theta$. Taking logarithms gives the release-count condition.

\paragraph{Proof of Proposition~\ref{prop:hybrid}.}
On the intersection of the simultaneous coverage events for the two
component procedures, both lower bounds are no larger than every
$\eta_\theta$; their maximum is therefore also no larger. A union bound
shows that the intersection fails with probability at most
$\alpha_{\mathrm N}+\alpha_{\mathrm D}\leq\alpha$.

\paragraph{Proof of \cref{thm:mechanism_collective}.}
Condition on the sampled releases, the fixed development choices, and the
resulting configuration of true and false release nulls. Under the
conditional super-uniformity and independence or PRDS assumptions, the Simes
partial-conjunction value
$\widehat p^{\mathrm{PC}}_{\theta,k}$ is super-uniform under
$H_{\theta,k}:K_\theta<k$
\citep{benjamini2008screening}. Monotonization can only increase this value.
If $\widehat K^{\mathrm{LB}}_\theta>K_\theta$, then the test for
$k=K_\theta+1$ rejected, so
\[
\Pr\!\left(
\widehat K^{\mathrm{LB}}_\theta>K_\theta
\mid Q_{\theta,1:R_\theta}
\right)
\leq
\alpha_{\mathrm R,\theta}.
\]
Integrating over the release draws preserves the bound, and the union bound
over mechanisms is at most $\alpha_{\mathrm R}$. As in
\cref{thm:mechanism}, if
$\widehat K^{\mathrm{LB}}_\theta\leq K_\theta$ for all $\theta$, then
replacing $V_{\theta}^{\mathrm{Holm}}$ by
$\widehat K^{\mathrm{LB}}_\theta$ in the binomial outer step cannot increase
false rejections. Holm at level $\alpha_{\mathrm M}$ across mechanisms then
contributes an additional $\alpha_{\mathrm M}$.

\paragraph{Proof of Proposition~\ref{prop:collective_dominance}.}
If $V_{\theta}^{\mathrm{Holm}}=v$, Holm testing on $\{p_{\theta,r}\}$ implies
that for every $k\leq v$,
\[
p_{\theta,(k)}
\leq
\frac{\alpha_{\mathrm R,\theta}}{R_\theta-k+1}.
\]
Taking the $j=k$ term in \cref{eq:pc_p} gives
\[
\widehat p^{\mathrm{PC}}_{\theta,k}
\leq
\left(R_\theta-k+1\right)p_{\theta,(k)}
\leq
\alpha_{\mathrm R,\theta}.
\]
This holds for every $k\leq v$, hence
$\widetilde p^{\mathrm{PC}}_{\theta,v}\leq
\alpha_{\mathrm R,\theta}$. Therefore
$\widehat K^{\mathrm{LB}}_\theta\geq V_{\theta}^{\mathrm{Holm}}$.

\paragraph{Proof of \cref{thm:mechanism_violation}.}
Let $K_\theta$ again be the number of truly valid releases. With probability
at least $1-\beta_{\mathrm R}$, the inner audit makes no false violation
declaration. On this event, $W_\theta\leq R_\theta-K_\theta$ and hence
$R_\theta-W_\theta\geq K_\theta$. Under
$H_{0,\theta}:\eta_\theta\geq\eta_0$, the lower-tail binomial value computed
from $K_\theta$ is super-uniform. The binomial cdf increases with its
observed count, so the value computed from $R_\theta-W_\theta$ is no smaller.
Holm therefore makes a false mechanism-violation declaration with
probability at most $\beta_{\mathrm M}$ when the inner event holds. A union
bound gives the stated result.

\paragraph{Release reporting versus mechanism power.}
The default procedure supports simultaneous statements about individual
releases and their mechanism, so it uses inner Holm correction. If only the
mechanism claim matters, \cref{thm:conditional_shared_target} reuses one target
through a registered release score and a contamination correction. More
adaptive scores, confidence-set inversion, or misclassification-aware
reliability models may improve power, but they need separate calibration. We
do not use uncorrected release labels as a shortcut.

\paragraph{Proof of Proposition~\ref{prop:adaptive}.}
Condition on the history before each fresh audit. For an invalid candidate,
$\Pr(p_t\leq\alpha_t\mid\mathcal F_{t-1})\leq\alpha_t$. Summing these
conditional probabilities over rounds gives the result.

\section{A sufficient audit size}
\label{app:sample_size}

\begin{proposition}[Best-case direct planning limits]
\label{prop:direct_planning}
Define
\[
d_\theta
=
\min_j d_{\mathrm B}\!\left(
\frac{\tau_j-s_j-a_j}{b_j-a_j}
\middle\Vert
\frac{\tau_j-a_j}{b_j-a_j}
\right).
\]
For one mechanism, direct validation at target $\eta_0$ is impossible unless
\[
n>
\frac{\log\{1/[\alpha_{\mathrm Z,\theta}(1-\eta_0)]\}}
{d_\theta}.
\]
At fixed $n$, let
$c_\theta=\exp(-nd_\theta)/\alpha_{\mathrm Z,\theta}$. If
$\eta_0+c_\theta<1$, validation is impossible unless
\[
R_\theta>
\frac{\log(\alpha_{\mathrm Q,\theta})}
{\log(\eta_0+c_\theta)}.
\]
\end{proposition}

These necessary limits assume every observed release scores one. They expose
the two costs separately: target records reduce contamination, whereas
mechanism draws tighten the binomial bound.

\paragraph{Planning takeaway.}
Target records and mechanism draws solve different problems. Under a common
margin, direct mode can reuse one target without paying the $\log R$ cost of
identifying every valid release, but it still needs enough releases to resolve
the reliability gap.

\begin{theorem}[Common-margin mode separation]
\label{thm:mode_separation}
Consider one mechanism with $J$ requirements of width at most $B$. Suppose
every valid release satisfies
$\mu_{Q,j}\leq\tau_j-\gamma$ for all $j$, and choose $0<s<\gamma$.
Let $R$ independent releases share $n$ i.i.d.\ target records. For any
$\delta_{\mathrm Z},\delta_{\mathrm Q}\in(0,1)$, the direct certificate obeys,
with probability at least $1-\delta_{\mathrm Z}-\delta_{\mathrm Q}$,
\begin{align}
L^{\mathrm{dir}}
\geq{}& \eta
-\frac{J\exp\{-2n(\gamma-s)^2/B^2\}}{\delta_{\mathrm Z}}
-\sqrt{\frac{\log(1/\delta_{\mathrm Q})}{2R}}
\nonumber\\
&-\sqrt{\frac{\log(1/\alpha_{\mathrm Q})}{2R}}
-\frac{\exp\{-2ns^2/B^2\}}{\alpha_{\mathrm Z}}.
\label{eq:direct_mode_separation}
\end{align}
By contrast, Hoeffding component tests followed by Holm recognize every valid
release with probability at least $1-\delta$ under the sufficient condition
\begin{equation}
n\geq
\frac{2B^2}{\gamma^2}
\max\!\left\{
\log\frac{R}{\alpha_{\mathrm R}},
\log\frac{RJ}{\delta}
\right\}.
\label{eq:named_mode_separation}
\end{equation}
Thus the target-side terms in the direct sufficient condition do not grow
with $R$, whereas simultaneous identification of the valid releases incurs a
$\log R$ target-sample term. Direct mode still needs enough releases to
resolve the reliability gap in \cref{eq:direct_mode_separation}.
\end{theorem}

\begin{proof}
For a valid release, a score failure requires at least one empirical mean to
exceed its true mean by $\gamma-s$. Hoeffding and a union bound give failure
probability at most
$J\exp\{-2n(\gamma-s)^2/B^2\}$. Markov's inequality bounds the
target-conditional valid-release miss mass by this quantity divided by
$\delta_{\mathrm Z}$. Conditional on that target, Hoeffding over the $R$
release scores loses the first square-root term. The exact binomial lower
bound is no smaller than the displayed Hoeffding lower bound, which loses the
second square-root term. Subtracting the registered false-pass allowance
gives \cref{eq:direct_mode_separation}.

For named mode, with probability at least $1-\delta$, every requirement of
every valid release has empirical margin at least $\gamma/2$ under the second
term in \cref{eq:named_mode_separation}. Its Hoeffding $p$-value is then at
most $\exp\{-n\gamma^2/(2B^2)\}\leq\alpha_{\mathrm R}/R$ by the first term.
All such release IUT values pass Holm's smallest threshold, so Holm recognizes
every valid release.
\end{proof}

For development planning we use $C(n,R)=c_Zn+c_QR$ and enumerate registered
$(n,R)$, slack, and error-allocation grids. We score each mode using pilot
estimates of its favorable-score or named-recognition probability. The planner
assigns the full audit budget to the selected mode before target access.
\Cref{tab:cost_planning} fixes budget 10,000 and compares a moderate regime
(direct score probability 0.95, named recognition 0.80) with a high-signal
regime (0.95 and 0.98). The planner selects direct mode throughout the former
and named mode throughout the latter, while increasing release cost reduces
the affordable $R$. These projected bounds are planning criteria, not
confidence statements about pilot estimates.

\begin{table}[H]
\centering
\small
\input{tables/cost_planning}
\caption{Cost-normalized development plans. One target record costs one unit;
$c_Q/c_Z$ varies by row.}
\label{tab:cost_planning}
\end{table}

We also froze a pilot-misspecification check. Holding the true planning rates
at 0.95 for direct scores and 0.80 for named recognition, we perturbed each
pilot input independently by
$\{-0.10,-0.05,-0.02,0,0.02,0.05,0.10\}$. The chosen mode changes in
6.1--8.2\% of the 49 perturbation pairs, while mean regret relative to the
oracle plan is 0.010--0.013; the worst adversarial pair loses about 0.165
(\cref{tab:cost_planning_sensitivity}). A pilot can therefore select the
wrong side of a close planning comparison, and the projected bound should be
stress-tested rather than reported as assured power.

\begin{table}[H]
\centering
\small
\input{tables/cost_planning_sensitivity}
\caption{Sensitivity of the development-only planner to misspecified pilot
probabilities. Regret is the oracle projected lower bound minus the achieved
projection under the fixed true rates.}
\label{tab:cost_planning_sensitivity}
\end{table}

\begin{proposition}[Preregistered hybrid]
\label{prop:hybrid}
Suppose named-release bounds are simultaneously valid outside an event of
probability $\alpha_{\mathrm N}$ and direct bounds outside an event of
probability $\alpha_{\mathrm D}$. If
$\alpha_{\mathrm N}+\alpha_{\mathrm D}\leq\alpha$, then
\[
L_\theta^{\mathrm{hyb}}
=
\max\{L_\theta^{\mathrm{named}},L_\theta^{\mathrm{dir}}\}
\]
satisfies
$\Pr\{\eta_\theta\geq L_\theta^{\mathrm{hyb}}\ \forall\theta\}\geq1-\alpha$.
\end{proposition}

\begin{proposition}[Sufficient size for named-release recognition]
\label{prop:named_sample_size}
Let $B=\max_j(b_j-a_j)$ and suppose every candidate has margin
$\mu_{Q,j}\leq\tau_j-\gamma$ for all $j$, with $\gamma>0$. Using
Hoeffding component tests and $n$ audit records per candidate, Holm validates
every candidate with probability at least $1-\beta$ when
\[
n
\geq
\frac{2B^2}{\gamma^2}
\max\!\left\{
\log\frac{M}{\alpha},
\log\frac{M|\mathcal J|}{\beta}
\right\}.
\]
\end{proposition}

\begin{proof}
Hoeffding's inequality and a union bound show that, with probability at least
$1-\beta$, every empirical mean is at most
$\mu_{Q,j}+\gamma/2$ under the stated sample size. The same condition
makes the component-test radius at level $\alpha/M$ no larger than
$\gamma/2$. Thus every component $p$-value, and hence every candidate IUT
$p$-value, is at most $\alpha/M$. Bonferroni rejects all candidate nulls, so
Holm does as well.
\end{proof}

\section{Losses and component tests}
\label{app:component_tests}

\subsection{Exact test in the Bernoulli simulation}

For loss count $K$ from $n$ audit records and tolerance $\epsilon$, the
controlled study tests
\[
H_0:\mu\geq\epsilon
\quad\text{against}\quad
H_1:\mu<\epsilon.
\]
The least favorable distribution under the null has parameter $\epsilon$,
which gives
\[
p=\Pr_{\operatorname{Binomial}(n,\epsilon)}(X\leq K).
\]
The candidate $p$-value is the maximum of its three component values.

\subsection{Target reuse stress test}

Every candidate in this registered simulation is invalid by the construction
used in Figure~\ref{fig:calibration}. The selection rule chooses the
smallest intersection--union $p$-value among 20 candidates. It then either
tests that candidate on the same target, tests it once on a fresh target, or
keeps the original target but corrects the full family.
\Cref{tab:target_reuse} reports the three resulting error rates.

\begin{table}[h]
\centering
\small
\setlength{\tabcolsep}{5pt}
\input{tables/target_reuse}
\caption{False validation after adaptive proxy selection. Every candidate is
invalid. We froze the registry and seed before the 5{,}000-trial run.}
\label{tab:target_reuse}
\end{table}

\section{Planning, adaptive search, and exposure claims}
\label{app:additional_procedures}

\begin{proposition}[Release planning]
\label{prop:release_planning}
Even when Holm recognizes every sampled release as valid, its first threshold
cannot validate one of $H$ mechanisms unless
\[
R_\theta
\geq
\frac{\log(\alpha_{\mathrm M}/H)}{\log(\eta_0)}.
\]
\end{proposition}
This follows from
$p^{\mathrm{mech}}_\theta=\eta_0^{R_\theta}$ when
$V_\theta=R_\theta$. Writing
$\alpha_{\mathrm R}=\lambda\alpha$ exposes a planning trade-off:
larger $\lambda$ helps recognize releases but leaves a stricter outer test.
We use $\lambda=1/2$ as a registered default; any power-based choice must
precede target access.

\paragraph{Candidates proposed over time.}
At round $t$, let the next candidate depend on earlier rounds and evaluate it
on a fresh target batch. If its $p$-value is conditionally super-uniform,
set
\[
\alpha_t=\frac{6\alpha}{\pi^2t^2},
\qquad
\sum_{t=1}^{\infty}\alpha_t=\alpha.
\]
\begin{proposition}[Adaptive candidate stream]
\label{prop:adaptive}
Validating candidate $t$ only when $p_t\leq\alpha_t$ gives
\[
\Pr\!\left(\text{an invalid candidate is ever validated}\right)\leq\alpha.
\]
\end{proposition}
The proof is in \appref{app:proofs}. Alpha spending does not make a target
reusable after its outcomes have shaped the next proxy.

\paragraph{Conditional exposure.}
An attack needs its own bounded estimand. We use finite-pool advantage,
$\mathrm{TPR}-\mathrm{FPR}$, at a registered threshold and apply the same
IUT and Holm construction to an upper ceiling. The result is conditional on
the release, record pools, representation, attack suite, and attacker seeds;
details are in \appref{app:attack}.

\section{Additional release-level calibration}
\label{app:release_results}

\begin{figure}[h]
\centering
\includegraphics[width=\textwidth]{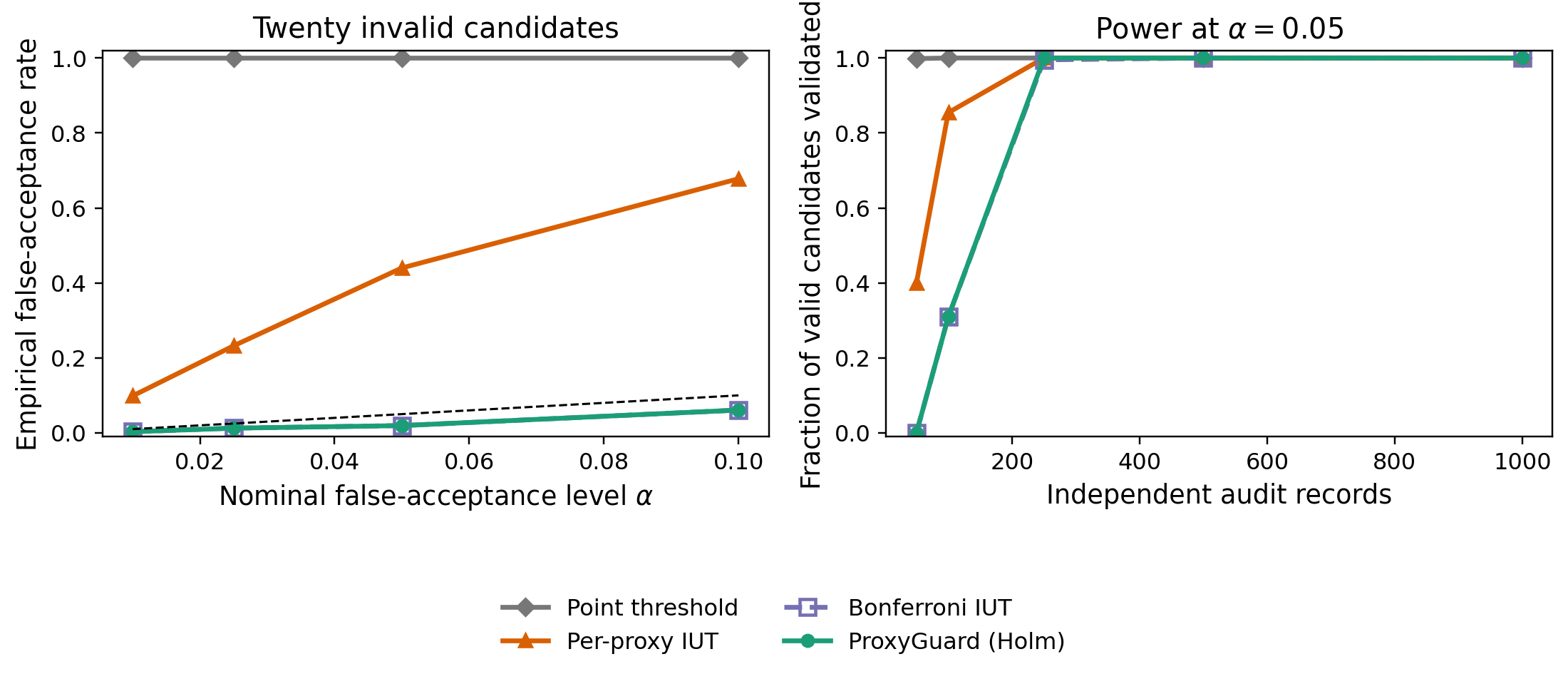}
\caption{Controlled release study. Left: the chance of validating at least
one of 20 invalid candidates. The dashed line is the requested level. Right:
the fraction of valid candidates validated as the audit sample grows.}
\label{fig:calibration}
\end{figure}

\Cref{fig:calibration} shows the Bernoulli calibration and power curves.
The held-out distribution families in \cref{tab:out_of_mechanism} test the
same decision rule under continuous, rare-subgroup, and correlated regrets.
\begin{table}[h]
\centering
\small
\setlength{\tabcolsep}{4pt}
\input{tables/out_of_mechanism}
\caption{Held-out regret distributions at $n=1{,}000$. FWER is the chance of
validating at least one boundary-invalid candidate. Power is the fraction of
valid candidates validated.}
\label{tab:out_of_mechanism}
\end{table}

\subsection{Prediction losses}

For binary label $y$ and predicted probability $p$, Brier loss is
$(p-y)^2$. Clipped log loss uses
$p_c=\min(1-\eta,\max(\eta,p))$ with $\eta=10^{-6}$. We divide it by
$-\log\eta$ to place it in $[0,1]$.

For false-negative cost $c$ and false-positive cost 1, normalized decision
loss is
\[
\ell_c(y,\widehat y)
=
\frac{
c\,\mathbf{1}\{y=1,\widehat y=0\}
+
\mathbf{1}\{y=0,\widehat y=1\}
}{\max(c,1)}.
\]
Each paired proxy-minus-source regret therefore lies in $[-1,1]$.

\section{Held-out regret distributions}
\label{app:heldout_regrets}

The continuous family draws regrets by shifting a beta distribution on
$[-1,1]$. The rare-subgroup family is a mixture: 8\% of records come from a
high-regret component and 92\% from a low-regret component. The correlated
family adds one record-level noise term shared by all candidates and a
smaller candidate--requirement term. In the invalid condition, one
requirement for every candidate has mean exactly at the tolerance. In the
power condition, every requirement is below tolerance by at least 0.04
before clipping.

We chose these families before their full 5{,}000-repetition runs. They
are not fitted to the Bernoulli simulation.

\subsection{Stratified rare-group audit}

The registered follow-up uses one subgroup-specific Bernoulli-risk
requirement with tolerance 0.10. The subgroup prevalence is 0.02, its valid
risk is 0.04, and the boundary risk is 0.10. We test six candidates with
Holm correction over 5{,}000 repetitions. Simple random sampling yields a
random subgroup count. The stratified design fixes the subgroup count in
advance and uses the remaining labeled records for the rest of the audit.
We registered this design before the run.

\begin{table}[h]
\centering
\small
\setlength{\tabcolsep}{1.8pt}
\input{tables/stratified_subgroup}
\caption{Registered subgroup-specific audit. FWER is familywise over six
boundary candidates; we average power over six valid candidates. Monte
Carlo standard errors are at most 0.31 percentage points. Stratified
sampling fixes the rare-group labeled count but does not include recruitment
or screening cost.}
\label{tab:stratified_subgroup}
\end{table}

\Cref{tab:stratified_subgroup} does not show that stratification is free. At prevalence 0.02,
obtaining 500 eligible rare-group records may require screening roughly
25{,}000 population records. It shows instead that a global random audit can
be uninformative for a registered rare-group claim even when the total labeled
sample appears large.

\section{Mechanism-level experiments}
\label{app:mechanism_experiments}

\Cref{tab:appendix_experiment_map} distinguishes the primary direct comparison
from calibration, planning, and historical studies.
The AIM and bootstrap mechanism details are collected in
\cref{app:mechanism_details}.

\begin{table}[h]
\centering
\small
\setlength{\tabcolsep}{5pt}
\input{tables/appendix_experiment_roadmap}\caption{Roadmap for the mechanism experiments in this appendix.}
\label{tab:appendix_experiment_map}
\end{table}

\subsection{Registered simulation}

The mechanism simulation has five registered mechanisms and three Bernoulli
requirements per release. A claim-valid release has risk 0.02 on every requirement.
A bad release has one risk 0.12 and two risks 0.02. The loss tolerance is
0.10. Boundary mechanisms produce a claim-valid release with probability 0.80;
the power condition uses probability 0.98. We evaluate
$R\in\{10,25,50,75\}$ releases, audit sizes
$n\in\{100,250,500\}$, and 5{,}000 repetitions. We split the 0.05 total error
budget equally between release and mechanism testing, then divide the release
budget equally across the five mechanism-specific families. Each
release has its own simulated audit batch, so the collective mode satisfies
the independence condition in \cref{thm:mechanism_collective}.

We fixed the registry, random seeds, reliability targets, and full design
before the simulation. After inspecting the first registered run, we added
the direct plug-in release-fraction baseline because it was missing from the
original comparison. A dated amendment records this addition. It does not
change the data-generating process or any \method{} result.
\Cref{tab:mechanism_calibration_full,fig:mechanism_control} report the
registered comparison at 50 releases and audit size 500.

\begin{table}[h]
\centering
\small
\setlength{\tabcolsep}{5pt}
\input{tables/mechanism_calibration}
\caption{Mechanism simulation with 50 releases and audit size 500. False
mechanism validation is familywise over five mechanisms with
$\eta=\eta_0=0.8$; power uses $\eta=0.98$.}
\label{tab:mechanism_calibration_full}
\end{table}

\begin{figure}[h]
\centering
\includegraphics[width=.92\textwidth]{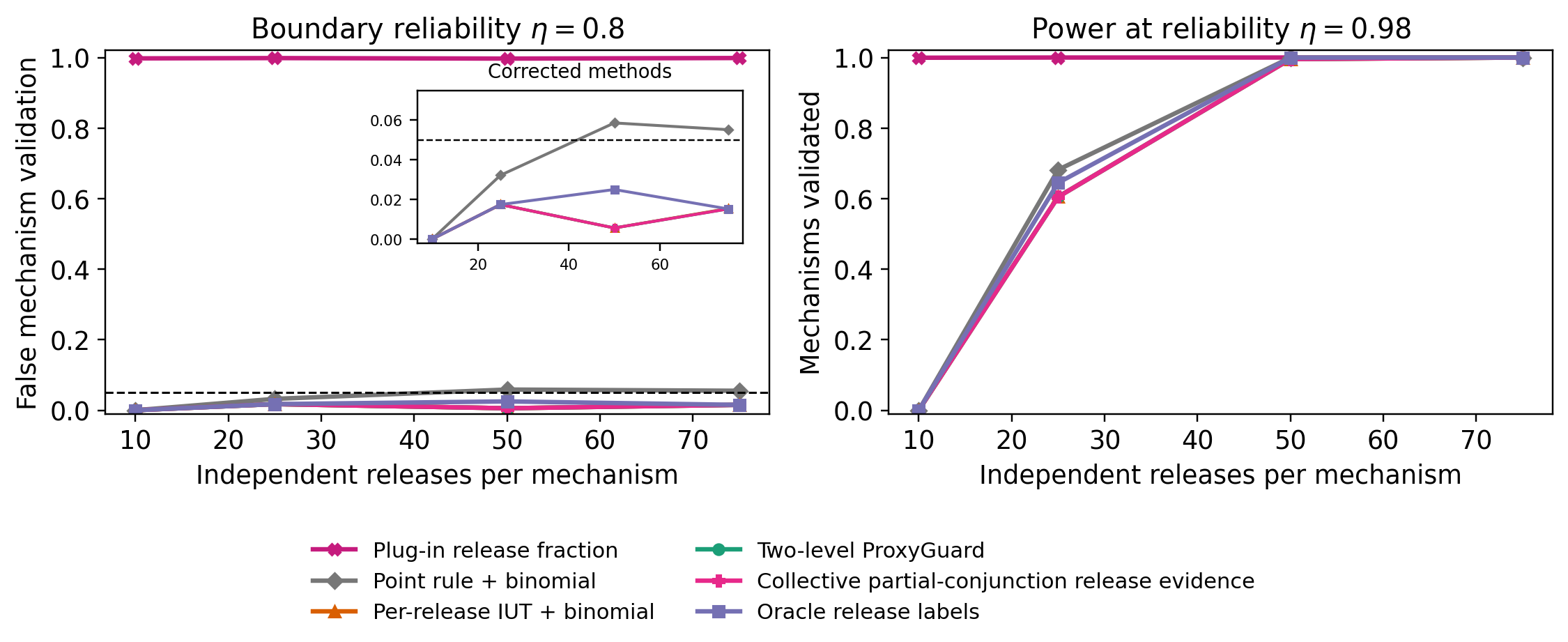}
\caption{Mechanism calibration at audit size 500. The left panel uses
boundary mechanisms; the right uses reliability 0.98. A plug-in passing
fraction carries no uncertainty correction.}
\label{fig:mechanism_control}
\end{figure}

\subsection{Direct shared-target confirmation}

This experiment checks calibration when shared target records make release
scores strongly dependent before conditioning.

The exploratory shared-target grid varied
$R\in\{250,1000\}$, $n\in\{500,1000,2000\}$, and
$\eta\in\{0.80,0.90,0.95\}$. We used it only to select a
multiplicity-limited confirmation cell. We then registered a confirmation
with a new seed and 1{,}000 repetitions.

The confirmation uses one Bernoulli requirement with limit 0.20. Conditional
on validity, release risk is uniform on $[0.135,0.145]$; otherwise it is
uniform on $[0.205,0.220]$. All 1{,}000 releases use the same 1{,}000 target
uniforms, which induces strong dependence among their empirical losses. The
direct score is one when empirical risk is at most 0.158. Its bounded-loss
ceiling is
\[
\kappa
=
\exp\{-1000\,d_{\mathrm B}(0.158\Vert0.2)\}
=
0.00291.
\]
The target-side allocation is 0.045, so the contamination allowance is
$\kappa/0.045=0.0647$; the remaining 0.005 covers conditional binomial
inference across releases. Named-release Holm uses 0.025 for release
identification and 0.025 for its outer reliability bound.
\Cref{tab:conditional_shared_target} reports the full result. This reversal
between $\eta=0.90$ and $\eta=0.95$ emerged in the frozen confirmatory run and
did not guide cell selection.

\begin{table}[h]
\centering
\small
\setlength{\tabcolsep}{6pt}
\input{tables/conditional_shared_target}
\caption{Validation rates (\%) over 1{,}000 repetitions in the initial
one-requirement shared-target confirmation. The boundary row measures false
validation; the remaining rows measure power. The direct boundary count is
shown explicitly; its exact one-sided 95\% upper rate is 0.30\%.}
\label{tab:conditional_shared_target}
\end{table}

\subsection{Primary confirmatory direct-mode comparison}

This comparison asks which mode has higher power when individual releases
provide moderate evidence and when they provide high-signal evidence.

We ran two further confirmations with the same three requirements: relative
score in $[-1,1]$ with limit 0.10, absolute risk in $[0,1]$ with limit 0.35,
and decision cost in $[0,2]$ with limit 0.80. Valid-release margins and
invalid-release excesses are continuous. All releases share the same target
fluctuation, and the requirements have additional correlated noise. The two
methods tune only their internal error allocation and the direct slack on the
same pilot. We use new seeds for confirmation.

The moderate-evidence design uses a smoothed Bernoulli loss: 98\% of each
normalized loss is a thresholded shared target uniform and 2\% is continuous
requirement-specific noise. Thus every loss is continuous and bounded, while
its variance remains large enough that individual release tests are
moderately informative. We fixed the pilot-selected allocations before 500
confirmatory repetitions. The separately registered high-signal design has
lower within-release variance. It supplies a
counterexample to universal direct-mode gains.

\Cref{fig:direct_multirequirement_power} reports direct minus named power at
reliability 0.95 over the full $(R,n)$ grids. In the moderate-evidence design,
direct inference gains 58.6 percentage points at $R=n=1{,}000$ and 49.6
points at $R=250,n=1{,}000$. With only 500 target records, its contamination
allowance is vacuous; with 2{,}000, named release evidence is often already
decisive. In the high-signal design, direct inference ties in nine of 45
$(R,n,\eta)$ cells and is weaker in the remaining 36. Across both studies,
the largest boundary false-validation rate among the proposed methods is
5.0\%. When no error occurs in 500 trials, the exact one-sided 95\% Monte
Carlo upper bound is 0.60\%.

\subsection{Why the contamination correction is necessary}
\label{app:false_pass_diagnostic}

The direct theorem does not treat a favorable empirical score as proof that a
release is valid. We tested the consequence of dropping that distinction in a
separate boundary experiment. Each mechanism has true reliability
$\eta=\eta_0=0.8$. Valid-release Bernoulli risks are uniform on
$[0.30,0.40]$; invalid risks are uniform on $[0.5005,0.54]$ against limit
0.50. All releases use the same target uniforms. An exploratory grid fixed
$R=250$, $n=1{,}000$, and slack 0.01 before a 2{,}000-trial confirmation with
a new seed. We registered the design before confirmation.

\begin{table}[h]
\centering
\small
\setlength{\tabcolsep}{6pt}
\input{tables/false_pass_diagnostic}
\caption{Boundary diagnostic for the direct contamination correction. The
score-only row uses the same conditional binomial lower bound as direct mode
but omits $\kappa/\alpha_{\mathrm Z}$. It is an intentionally invalid
diagnostic, not a proposed method.}
\label{tab:false_pass_diagnostic}
\end{table}

\Cref{tab:false_pass_diagnostic} shows that the uncorrected score-only bound
validates 161 of 2{,}000 boundary
mechanisms, or 8.05\%; its exact one-sided 95\% Monte Carlo upper bound is
9.12\%. Invalid releases contribute an average 1.21 percentage points to the
observed score frequency. The registered Chernoff ceiling is 0.819, so the
correct contamination allowance is vacuous in this small-slack regime and
the direct method abstains in every trial; the exact one-sided 95\% upper
bound on its false-validation rate is 0.15\%. The oracle-label reference
validates 4.00\%. This stress test shows why
$\underline\pi_\theta^{\mathrm{CP}}$ alone is not a reliability bound.

\subsection{Smooth target-concentration confirmation}
\label{app:smooth_target_confirmation}

Theorem~\ref{thm:smooth_target_concentration} uses a ramp score to replace
the expectation-only conversion by bounded target sensitivity. We tested a
one-requirement shared-target mechanism with limit 0.8, slack 0.05, valid risk
0.1, and boundary-invalid risk 0.8. A first pilot sought a separation at
reliability 0.95 but found no eligible cell, so we ran no confirmation from
that registry. A second exploratory registry asked the narrower question at
reliability 0.99 and fixed $R=500$, $n=200$, and ramp width 0.6 before a
2{,}000-trial confirmation with a new seed. We registered this revised design
before confirmation.

\begin{table}[h]
\centering
\small
\setlength{\tabcolsep}{6pt}
\input{tables/smooth_target_concentration}
\caption{Validation rates (\%) in the smooth-score confirmation. The first
column is the reliability boundary. Both expectation-only rows use the same
error split as the smooth concentration certificate.}
\label{tab:smooth_target_concentration}
\end{table}

The smooth invalid-release ceiling is 0.00544. Markov turns it into 0.2174;
the additive target radius is 0.1601, giving a total allowance of 0.1655.
As \cref{tab:smooth_target_concentration} reports, the new certificate
validates 1{,}729 of 2{,}000 mechanisms at reliability
0.99 while both expectation-only certificates validate none. It has no false
validation in 2{,}000 boundary trials (exact one-sided 95\% upper rate
0.15\%). At reliability 0.95 all three feasible-data certificates remain
unresolved. This is a high-reliability, large-margin stress test, not evidence
that smooth concentration dominates in the moderate-evidence regime.

\subsection{Collective-evidence grid}

This grid asks when independent-batch partial conjunction improves on
identifying releases one at a time.

We registered the collective extension before running the full grid. It uses five
mechanisms, three Bernoulli requirements, tolerance 0.10, and a bad-release
risk of 0.12. The registered reliability target is $\eta_0=0.8$. For each
cell, a boundary draw with $\eta=0.8$ estimates familywise false mechanism
validation. An independent draw with
$\eta\in\{0.90,0.95,0.98\}$ estimates power. We vary valid-release risk over
$\{0.04,0.06,0.08\}$, target audit size over
$\{100,250,500,1000\}$, release count over $\{25,50,100\}$, and true
mechanism reliability in the power draw over
$\{0.90,0.95,0.98\}$. Every cell has 2{,}000 repetitions. Independent
Bernoulli audit batches accompany every release, so a cell uses $nR$
audit observations per mechanism. This experiment does not test the
collective procedure under shared target records.

\Cref{fig:collective_power} reports the complete power-gain surface at
reliability 0.95 rather than selecting one favorable cell. The strongest gain
appears at valid-release risk 0.04 and audit size 500. At audit size 1{,}000,
individual release evidence is already decisive when risk is 0.04, so the
collective procedure adds no power. At risk 0.08, both procedures remain
underpowered throughout this grid. The largest observed familywise false
mechanism-validation rate over the boundary draws and both corrected methods
is 0.018.

\begin{figure}[h]
\centering
\includegraphics[width=.92\textwidth]{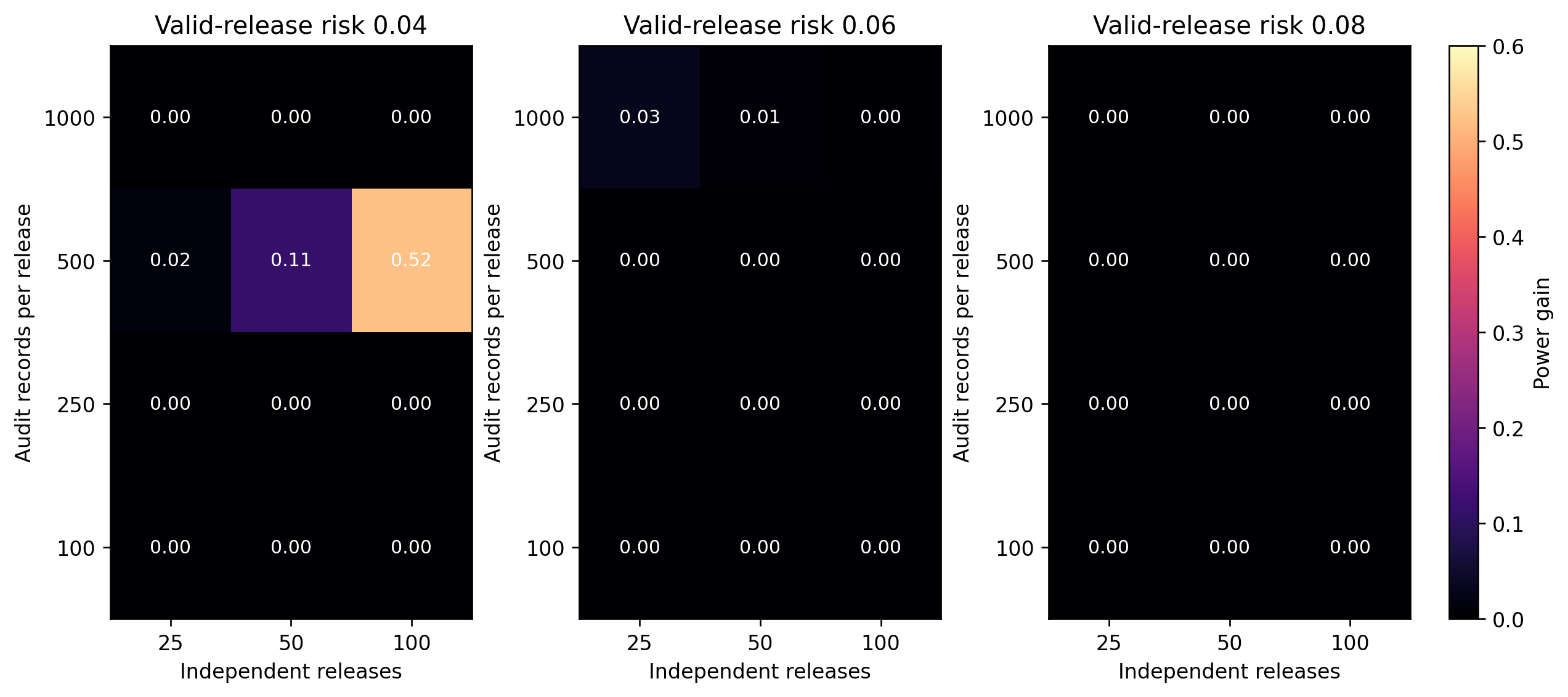}
\caption{Power gained by tail-Simes over named-release Holm with independent
target batches. Releases have reliability 0.95. Each cell reports the power
difference; $nR$ is the target-observation cost per mechanism.}
\label{fig:collective_power}
\end{figure}

After inspecting the registered tail-Simes results, we added a tail-Fisher
partial-conjunction benchmark without changing the data-generating process,
seed, or grid. We documented this addition as a post hoc amendment.
\Cref{tab:collective_benchmarks} gives the selected moderate-evidence cell.
Parentheses contain Monte Carlo standard errors computed across the 2{,}000
repetitions. Fisher is more powerful here, but the comparison does not support
a confirmatory ranking of combining rules.

\begin{table}[h]
\centering
\small
\setlength{\tabcolsep}{4pt}
\input{tables/collective_benchmarks}
\caption{Collective-method comparison at valid-release risk 0.04,
$n=500$, $R=100$, and power reliability 0.95. Boundary FWER uses
$\eta=\eta_0=0.8$. We registered tail-Simes; tail-Fisher is a post hoc
benchmark.}
\label{tab:collective_benchmarks}
\end{table}

\subsection{Near-boundary inner correction}

This follow-up tests whether release-level multiplicity correction still
matters when invalid releases lie close to the risk limit.

This registered follow-up uses five mechanisms, 200 releases per mechanism,
three requirements, and 10{,}000 repetitions at audit size 500. The true
mechanism reliability is on the 0.8 boundary. A claim-valid release has risk 0.02
throughout; a bad release has one risk in
$\{0.100,0.101,0.102,0.120\}$ against tolerance 0.10. ``Separate release
IUTs'' tests every release at $\alpha_{\mathrm R}=0.025$ without correcting
across releases. The inner-Holm rule corrects the same family. Both feed
their recognized-release counts to the same outer binomial tail at
$\alpha_{\mathrm M}=0.025$.

\begin{table}[h]
\centering
\small
\setlength{\tabcolsep}{4.5pt}
\input{tables/near_boundary_inner_correction}
\caption{Near-boundary release recognition. ``Any false release label'' is
familywise over all releases from five boundary mechanisms. The last column
is familywise over the five outer mechanism decisions.}
\label{tab:near_boundary_inner}
\end{table}

\Cref{tab:near_boundary_inner} shows that the uncorrected rule makes at least
one false release recognition in most
near-boundary trials. Inner Holm nearly removes this error. The outer
binomial tail is less sensitive because a small number of false release
labels seldom moves the count past its rejection threshold. This experiment
therefore supports the need for inner correction when reporting release-level
decisions, but does not show a large gain in mechanism-level calibration.
Monte Carlo standard errors for the displayed rates are at most 0.25
percentage points.

\subsection{Allocation of the error budget}

We varied $\lambda$ in
$\alpha_{\mathrm R}=\lambda\alpha$ over a registered grid. The setting uses
five mechanisms, 25 releases, audit size 500, bad-release risk 0.102, and
10{,}000 repetitions. We did not select a split after seeing the outcomes.

\begin{table}[h]
\centering
\small
\setlength{\tabcolsep}{4.5pt}
\input{tables/alpha_allocation}
\caption{Sensitivity to the release-layer share of the total 0.05 error
budget. False validation is familywise over five boundary mechanisms; power
uses reliability 0.98.}
\label{tab:alpha_allocation}
\end{table}

\Cref{tab:alpha_allocation} shows that the outer test becomes the bottleneck
when $\lambda$ is large. At 25
releases, $\lambda\geq0.75$ leaves too little outer budget for any mechanism
to cross the first Holm threshold. The observed powers at $\lambda=0.10$ and
$0.25$ are 64.51\% and 64.34\%; their Monte Carlo standard errors are about
0.48 percentage points. This is a planning result for one design, not a
universal allocation rule.

\subsection{Release-count planning}

Proposition~\ref{prop:release_planning} gives a necessary release count before target
sample size enters the calculation. The values below use
$\alpha_{\mathrm M}=0.025$ and the first Holm threshold
$\alpha_{\mathrm M}/H$.

\begin{table}[h]
\centering
\small
\setlength{\tabcolsep}{5pt}
\input{tables/mechanism_planning}
\caption{Minimum number of recognized-valid releases needed when Holm
recognizes every sampled release as valid.}
\label{tab:mechanism_planning}
\end{table}

\Cref{tab:mechanism_planning} makes the release-count bottleneck explicit:
even perfect inner recognition cannot overcome too few mechanism draws.

\subsection{Adaptive candidate stream}

At round $t$, the boundary candidate has one Bernoulli risk equal to 0.10
and two equal to 0.02. A valid candidate has all three risks equal to 0.02.
Every round receives 250 new target records. We compare testing each round
at 0.05 with the schedule
$\alpha_t=6(0.05)/(\pi^2t^2)$ over 5{,}000 repetitions. The main text
reports cumulative false validation and the probability of validating a
valid candidate that first appears in a given round.

\begin{table}[h]
\centering
\small
\setlength{\tabcolsep}{5pt}
\input{tables/adaptive_search}
\caption{Adaptive search with fresh target batches. Power is for a valid
candidate arriving in the named round under quadratic spending.}
\label{tab:adaptive_search}
\end{table}

\Cref{tab:adaptive_search} reports the cost of long candidate streams:
fixed-level testing accumulates false validations, whereas spending preserves
the registered error budget but reduces power in later rounds.

\subsection{AIM and bootstrap mechanisms}
\label{app:mechanism_details}

The AIM mechanism analysis groups the 45 existing releases into nine
dataset--privacy configurations. We fixed $\eta_0=0.8$ after those release
audits existed, so this analysis is retrospective. Each configuration has
only five releases. The table reports separately valid 95\% simultaneous
one-sided lower and upper bounds in \cref{tab:aim_mechanisms}.

\begin{table}[h]
\centering
\small
\setlength{\tabcolsep}{3pt}
\input{tables/aim_mechanism_reliability}
\caption{Two-level retrospective sensitivity analysis of nine AIM
configurations at reliability target $\eta_0=0.8$. Release counts are
validated/unresolved/violation. The displayed bounds include inner release
uncertainty and outer multiplicity.}
\label{tab:aim_mechanisms}
\end{table}

\paragraph{Informed AIM same-table replication.}
The revision registry presampled 25 AIM seeds for Taiwan at $\epsilon=1$.
The first seed also fixed the source, validation, and target split. Before
inspecting any release losses or audit decisions, we recorded an
amendment excluding that release from the binomial mechanism sample.
Conditional on the frozen split, releases 2--25 are independent draws from
the registered seed law; release 1 remains descriptive. The effective outer
sample therefore has 24 releases. We chose this configuration because it
failed in the earlier retrospective analysis, so the study is an informed
replication rather than a new generator search. The new split did not create
an untouched audit reserve, as documented below.

\begin{table}[h]
\centering
\small
\setlength{\tabcolsep}{4pt}
\input{tables/prospective_aim_mechanism}
\caption{Informed AIM same-table replication after the pre-outcome amendment.
Release-count columns list validated/unresolved/violation. The upper endpoint is a
nominal 95\% simultaneous one-sided bound; record overlap prevents a
theorem-backed prospective interpretation.}
\label{tab:prospective_aim}
\end{table}

As \cref{tab:prospective_aim} shows, all 24 eligible releases have a
corrected violation. Their mean AUC change is
$-0.203$ and mean normalized Cost5x regret is $+0.035$. The mechanism-level
violation $p$-value is $1.68\times10^{-17}$, and the simultaneous reliability
upper bound is 0.142 at the separately registered 0.05 violation-error
budget. These calculations describe the registered same-table rerun but do
not inherit the guarantee in \cref{thm:mechanism_violation}.

We registered the first bootstrap control before release generation. It
draws 30 training tables by sampling Taiwan Default source rows with
replacement at the original sample size. Each release chooses its learning
pipeline on development data, and all releases use the sealed target audit.
Mean Brier, clipped-log-loss, and Cost5x regrets are 0.0036, 0.0008, and
0.0059. Their largest simultaneous release upper bounds are 0.0164, 0.0095,
and 0.0281. At the registered 0.01 limits, the procedure neither validates nor
invalidates any of the 30 releases. Both one-sided mechanism bounds are therefore
vacuous.

After that result, we registered an informed positive-control replication.
It uses a new split and 30 new bootstrap seeds. The relative-regret limits
are 0.04. We also register absolute proxy-risk limits of 0.18 for Brier,
0.07 for normalized clipped log loss, and 0.16 for normalized Cost5x.
\Cref{tab:positive_bootstrap} reports the release-level evidence.

\begin{table}[h]
\centering
\small
\setlength{\tabcolsep}{4pt}
\input{tables/positive_bootstrap_requirements}
\caption{Informed bootstrap same-table control. ``Largest upper bound'' is
the largest simultaneous release-level upper bound among 30 releases. All
six limits came from the pilot, and we fixed them before evaluating the new split.}
\label{tab:positive_bootstrap}
\end{table}

All 30 releases validate all six requirements. The outer validation
$p$-value is 0.00124 and the 95\% simultaneous one-sided reliability lower
bound is 0.884, above the registered 0.8 target. We fixed the wider limits
from the earlier pilot before evaluating outcomes for the new split. The
algorithm returns a positive decision on this nonprivate high-fidelity
control, but record reuse prevents a prospective reliability claim. It is
not evidence about synthetic-data privacy.

\begin{table}[h]
\centering
\small
\setlength{\tabcolsep}{5pt}
\input{tables/positive_bootstrap_baselines}
\caption{Target risks used to interpret the absolute ceilings in the
informed positive control. Constant policies choose their decision threshold
on validation Cost5x before researchers open the target.}
\label{tab:positive_bootstrap_baselines}
\end{table}

The baselines in \cref{tab:positive_bootstrap_baselines} show that the
absolute ceilings are safeguards, not proposed deployment standards. Taken
alone, they are loose enough for the training-prevalence policy.
The registered claim is the conjunction of absolute and relative limits:
the constant-0.5 policy exceeds the 0.04 relative limits for Brier and
Cost5x, while the training-prevalence policy exceeds the Cost5x limit.
The source target risks and the simple policies make the strength of the
positive-control claim explicit.

\begin{table}[h]
\centering
\small
\setlength{\tabcolsep}{4pt}
\input{tables/audit_lineage}
\caption{Record lineage for the informed Taiwan reruns. The three overlap
columns partition each replication audit set: every audit record had already
appeared in the pilot split. The pilot seed was 3609; the AIM and bootstrap
rerun seeds were 1155954205 and 4921. Neither rerun has a record-disjoint
target reserve.}
\label{tab:audit_lineage}
\end{table}

We reconstructed \cref{tab:audit_lineage} from the registered split seeds
and stable dataset row positions. We selected the AIM configuration from the
pilot outcome, and the pilot informed the bootstrap limits and
positive-control design. Both reruns were therefore pilot-informed. A fresh
split seed changes the role assigned to a record, but it cannot make a
previously analyzed finite table independent of those choices. The repository
includes the reproducible record-level mapping with the lineage audit.

\section{Claim status across experiments}
\label{app:claim_status}

\begin{table}[!htbp]
\centering
\scriptsize
\input{tables/claim_status}\caption{Provenance and formal status of every reported experiment.
An untouched target is independent, across the project, from choices
informed by earlier analyses; a different split seed is not sufficient.
The guarantee column is ``Yes'' only when the design supports the conditions
of Assumption~\ref{ass:audit} or the registered stratified extension in
Proposition~\ref{prop:stratified_target}.}
\label{tab:claim_status}
\end{table}

\Cref{tab:claim_status} separates theorem-backed evidence from descriptive and
same-table analyses. We describe relative-only real-data decisions as meeting the registered
relative-transfer requirements. They do not establish absolute target
adequacy. The same distinction applies to rows labeled ``validated'' in
historical output files.
The prospective protocols are in
\cref{app:magic_sealed,app:spambase_sealed,app:neural_audits,app:rice_tvae,%
app:mushroom_direct,app:text_direct}.

\section{Prospective MAGIC sealed audit}
\label{app:magic_sealed}

MAGIC Gamma Telescope contains 19{,}020 simulated event records with ten
real-valued features \citep{bock2004magicuci}. We used the official UCI
download. Reserve seed 27072701 selected 5{,}706 row positions uniformly without replacement
before we parsed the data; this selection did not use class labels. The
remaining 13{,}314 records formed the development table. Its labels were
used to select the source pipeline and to set the six numerical limits.

The frozen protocol specifies development seed 27072702, release
seeds 127001--127030, relative limits of 0.04, reliability target 0.8, and
the separate 0.05 release/mechanism error budget. The three absolute
ceilings equal source validation risk plus 0.05, rounded upward to three
decimals: 0.141 for Brier, 0.072 for normalized clipped log loss, and 0.108
for normalized Cost5x. Before reading target outcomes, the audit script
checks that the protocol, development data, sealed target, and row assignments
match the frozen study records.

Hadron background is the positive class, so the registered 5:1
false-negative cost penalizes accepting a background event as gamma signal,
the asymmetric error described in the dataset documentation. All 30
bootstrap releases validate all six requirements. The mechanism lower bound
is 0.884 and its Holm-adjusted validation $p$-value is 0.00124. This is a
prospective fidelity result for the frozen nonprivate bootstrap mechanism;
it is not a privacy claim or evidence about a different generator.
\Cref{tab:magic_sealed} reports the six component results.

\begin{table}[!htbp]
\centering
\small
\setlength{\tabcolsep}{3.5pt}
\input{tables/magic_sealed_requirements}
\caption{Prospective MAGIC bootstrap audit. Means are across 30 releases;
``max UCB'' is the largest simultaneous release-level upper bound.}
\label{tab:magic_sealed}
\end{table}

\section{Prospective Spambase--AIM sealed audit}
\label{app:spambase_sealed}

Spambase contains 4{,}601 email records with 57 continuous attributes
\citep{hopkins1999spambase}. We used the official UCI archive.
Reserve seed 940728 selected 1{,}840 row
positions before we parsed any outcome; the remaining 2{,}761 records formed
the development table. We reverse the UCI indicator so the positive class is
legitimate email. The registered 5:1 false-negative cost therefore penalizes
a legitimate message classified as spam, following the asymmetry described
in the dataset documentation.

The partition protocol fixes 30 release seeds, AIM privacy parameters
$(\epsilon,\delta)=(5,10^{-6})$, a relative-risk tolerance of 0.06, the rule
for setting absolute ceilings, reliability target 0.8, and separate 0.05
budgets for false validation and false violation. Each mechanism draw
retrained AIM and sampled a fresh table. After a full 57-feature fit failed
to finish within 20 minutes, we stopped the run before the sealed target was
read. The 20-minute limit was not preregistered, and the incomplete run
produced no finished proxy release or proxy-utility result. We had observed
only development-side source quantities, which the design permits before the
target audit. To avoid an outcome-guided feature search, the amendment used
the first 12 attributes in the official schema as a deterministic prefix.
A computational amendment restricted the mechanism to the first 12
features, six bins, maximum model size 20, and 20 marginals. It left the
sealed reserve, release seeds, privacy parameters, claim-setting rule,
reliability target, and error allocation unchanged. We froze the amended audit
protocol after development and before target access.

The selected source procedure is a random forest with decision threshold
0.33. Its target AUC is 0.899 and normalized Cost5x risk is 0.041. Across the
30 AIM releases, mean proxy AUC is 0.535 and mean Cost5x risk is 0.100. Every
release has corrected violations of both Brier requirements. The
mechanism-level violation $p$-value is $1.07\times10^{-21}$ and the 95\%
simultaneous one-sided reliability upper bound is 0.116. This is a
prospective negative result for the registered Spambase--AIM configuration.
It does not transfer to another AIM workload, dataset, schema, privacy
setting, target population, or seed law.
The registered component results appear in \cref{tab:spambase_aim}.

All 30 DP fits remained internal audit computations on a public benchmark;
the disclosed result is the audit decision and summary statistics, not 30
synthetic tables. On a sensitive source, every released DP artifact would
require composition accounting. Thirty independent
$(5,10^{-6})$ training runs are not equivalent to one such release. Sampling
30 tables from one fitted private generator would avoid repeated
training-data access by post-processing, but would estimate reliability
conditional on that fitted generator rather than reliability of the complete
retraining mechanism.

\begin{table}[!htbp]
\centering
\small
\setlength{\tabcolsep}{3.5pt}
\input{tables/spambase_aim_requirements}
\caption{Prospective Spambase--AIM component audit. Means are across 30
independent full-pipeline releases; ``Viol.'' counts releases with a
corrected requirement-level violation.}
\label{tab:spambase_aim}
\end{table}

\section{Prospective full-pipeline neural audits}
\label{app:neural_audits}

We ran the same frozen full-pipeline design on three project-new UCI tables:
CDC Diabetes Health Indicators, Covertype, and Online Shoppers
\citep{uci2017cdcdiabetes,blackard1998covertype,sakar2018onlineshoppers}.
We partitioned each table into a source table, a development-only planning
set, and a record-disjoint target reserve before computing any target result.
The reserve fixes equal class counts, so its estimand is the balanced
mixture
$T_{\mathrm{bal}}=\tfrac12T(\cdot\mid Y=0)+
\tfrac12T(\cdot\mid Y=1)$, not natural-prevalence risk. After a review exposed
that distinction, we discarded the original i.i.d.\ analysis and applied the
stratified procedure in Proposition~\ref{prop:stratified_target} with weights
$(1/2,1/2)$. No target outcome, limit, slack, mechanism, or seed changed.
Development selected high-fidelity, moderate, and degraded settings from a
fixed nine-configuration grid. The target remained sealed until every
lightweight random-feature autoencoder and downstream proxy model had been
fitted and frozen. Thus each release redraws generator fitting, sampling, and
proxy learning rather than conditioning on one fitted generator.

The primary requirements are classification error, Brier loss, and normalized
5-to-1 decision cost, each in $[0,1]$. A primary limit is the source
procedure's planning risk plus a pre-target degradation budget, capped by a
fixed absolute ceiling. The dataset-specific budgets were 0.08 for CDC, 0.06
for Covertype, and 0.09 for Online Shoppers. They are audit policy choices,
not estimates of an externally mandated clinical or operational tolerance.
Direct scoring used the separately
registered cutoff $\tau_j-s_j$; the slack never changed $\tau_j$. The
reliability targets were 0.8 for CDC and Covertype and 0.85 for Online
Shoppers. We allocated error globally over the three configurations in each
study.

\begin{table}[H]
\centering
\small
\setlength{\tabcolsep}{4pt}
\input{tables/neural_full_pipeline_audits}
\caption{Corrective stratified analysis of the prospectively collected
lightweight-neural audits. $V/W$ gives named validations and detected
violations; ``Scores'' counts favorable direct scores.}
\label{tab:neural_full_pipeline}
\end{table}

\Cref{tab:neural_full_pipeline} reports every corrected configuration. The
sampling-valid named analysis recognizes 295 of 300 moderate Covertype
releases and gives $L^{\mathrm{named}}=0.948$, above the direct bound 0.885.
The former 0.827 named bound came from treating the fixed balanced reserve as
i.i.d.\ and is not retained. CDC high and moderate also validate under both
modes. All degraded configurations have mechanism-level violation evidence;
Online Shoppers high and moderate remain unresolved. No target result revised
a configuration, limit, slack, or reliability target.

\begin{table}[H]
\centering
\scriptsize
\setlength{\tabcolsep}{3.5pt}
\begin{tabularx}{\textwidth}{@{}lX@{}}
\toprule
Component & Frozen implementation \\
\midrule
Input map & source-min--max scaling; registered integer/categorical encodings \\
Encoder & one random $\tanh$ layer; 4--48 hidden features \\
Training & Gaussian corruption (0.01--0.35); ridge reconstruction solve ($10^{-3}$--0.25) \\
Sampling & balanced labels; class-matched source anchors; Gaussian perturbation; decoder clipping and rounding \\
Configuration grid & nine settings jointly vary width, corruption, sampling noise, label fidelity, and ridge penalty \\
Redrawn per release & hidden weights, corruption, anchors, synthetic table, and proxy fit seed \\
Downstream procedure & standardized logistic regression; fixed 0.5 and $1/6$ decision thresholds \\
\bottomrule
\end{tabularx}
\caption{Architecture of the lightweight neural generator. This is a
random-feature denoising autoencoder, not an adversarial or diffusion
synthesizer.}
\label{tab:neural_architecture}
\end{table}

\Cref{tab:neural_architecture} records the frozen architecture. Release
generation took 14 seconds for 500 CDC draws, 162 seconds for 500
Covertype draws, and 4.9 seconds for 700 Online Shoppers draws on the audit
machine. Target evaluation took 6.3, 6.9, and 9.0 seconds, respectively.
These timings exclude the one-time public-data download and expose the
mechanism-cost axis used by the planner in \cref{tab:cost_planning}.

\section{Prospective Rice--TVAE full-pipeline audit}
\label{app:rice_tvae}

The project-new Rice table contains 3{,}810 grains, seven measured shape
features, and Cammeo/Osmancik labels \citep{cinar2019rice}. Before development,
we fixed 500 source rows, 1{,}000 planning rows, and a record-disjoint
2{,}310-row empirical target population. The audit contains 2{,}000 i.i.d.\
draws with replacement from that sealed population, so its target statement
does not rely on the stratified correction.

Development compared standard generators from the official \texttt{ctgan}
package \citep{xu2019modeling}. Early CTGAN fits and short TVAE fits did not
meet the fixed policy limits. Because the target was still unopened, recorded
amendments expanded optimization and synthetic sample size, then froze a
replicated TVAE selection grid. Each mechanism draw refits TVAE, samples a
new table, and refits standardized logistic regression. The high and moderate
configurations use a 64-dimensional latent representation, two 128-unit
encoder and decoder layers, batch size 500, 2{,}000 synthetic rows, and 200 or
80 epochs. The degraded control trains on 300 rows after 20\% label
corruption. We fixed all release seeds before target access.

The primary limits are 0.151 classification error, 0.1344 Brier loss, and
0.112 normalized 5:1 cost. Each is the source planning risk plus a common
0.08 degradation allowance, capped by a pre-target absolute ceiling. Thus the
error allowance has the direct interpretation of at most eight additional
errors per 100 records; the Brier and normalized-cost allowances use the same
bounded-risk scale. The separate score slack is 0.035, and
$\eta_0=0.75$. One degraded draw produced a single-class table. A
pre-target compute amendment retained that seed and assigned unit loss to all
three requirements rather than silently replacing the failed release.

\begin{table}[H]
\centering
\small
\setlength{\tabcolsep}{3.5pt}
\input{tables/rice_tvae_audit}
\caption{Prospective Rice--TVAE audit. $V/W$ gives named validations and
detected violations. All bounds are simultaneous one-sided 95\% statements
over the three registered configurations.}
\label{tab:rice_tvae}
\end{table}

\Cref{tab:rice_tvae} shows that the high-fidelity mechanism validates under
both modes; direct increases the
lower bound from 0.852 to 0.864. The moderate mechanism remains unresolved,
although direct raises its bound from 0.405 to 0.428. The degraded mechanism
has a named violation-side upper bound of 0.375. Generation took 2{,}103
seconds and target evaluation 6.8 seconds. This is a prospective result for a
recognized standard generator with retraining randomness inside the
mechanism, but it still does not furnish a direct-only decision.

\section{Prospective Secondary Mushroom direct audit}
\label{app:mushroom_direct}

The official Secondary Mushroom archive contains simulated edible and
poisonous mushrooms with 20 measured attributes
\citep{wagner2021secondarymushroom}. We fixed a 10{,}000-record reserve by
uniform row sampling before parsing any outcome. The remaining 51{,}069 rows
formed the development partition. The study uses the 17 categorical features;
we excluded the three continuous features before creating the reserve.
We use one-hot logistic regressions for the source and proxy procedures and
select their Cost5x thresholds on development data.

Two amendments occurred before target access. First, a development-only parse
found an undocumented category. The first amendment maps any value outside the
official domain to the already registered unknown label rather than adding an
observed category. Second, the registered AIM pilot failed to complete within
600- and 1{,}800-second infrastructure allowances and produced no model or
pilot output. The reserve remained unopened. The second amendment replaced
that mechanism with a class-conditional categorical generator. The partition,
release seeds, learners, claims, slacks, reliability target, and error budgets
did not change.

The private fit releases one class histogram and one feature-by-class
contingency table for each of 17 features. One record changes one count in each
of these 18 tables, so under add-or-remove-one adjacency the vector query has
$\ell_1$ sensitivity 18. Independent Laplace noise with scale $18/20=0.9$
gives pure $\epsilon=20$ differential privacy. Nonnegative clipping, unit
pseudocount smoothing, and the 100 sampled
tables are post-processing. Reliability is therefore conditional on this one
private fit; it does not include retraining randomness.

The development pilot contained 20 post-processing releases. The registered
rule set each score cutoff to the largest pilot mean plus 0.005, rounded
upward to 0.001, then added the fixed direct slack to obtain the validity
limit. The frozen component values and audit means are in
\cref{tab:mushroom_requirements}.

\begin{table}[H]
\centering
\small
\setlength{\tabcolsep}{4pt}
\input{tables/secondary_mushroom_requirements}
\caption{Prospective Secondary Mushroom direct audit. The range is over 100
post-processing releases evaluated on the same 10{,}000 target records.}
\label{tab:mushroom_requirements}
\end{table}

All 100 releases score below all four cutoffs. The one-sided conditional score
lower bound is 0.948396. The invalid-release score ceiling is 0.000323 and the
registered target-side allocation converts it to a 0.007180 contamination
allowance. Thus $L^{\mathrm{dir}}=0.941216>\eta_0=0.8$. Named-release mode also
recognizes all 100 releases and gives a 0.963783 lower bound. Source AUC is
0.938 and mean proxy AUC is 0.869. The private fit took 0.03 seconds; sampling
100 tables and fitting their downstream procedures took 24.9 seconds on the
recorded machine. These costs describe this categorical post-processing
mechanism, not repeated private or neural training.

\section{Descriptive stand-in audits}
\label{app:standin_audits}

These examples have no row correspondence. Bank Marketing uses 2008 as
source, 2009 as proxy, and 2010 as target \citep{moro2014bankdataset}. Heart
trains on Cleveland and Hungary and audits on Switzerland and VA Long Beach
\citep{janosi1989heart}. Adult compares full and reduced schemas on disjoint
cohorts \citep{becker1996adultdataset}.

\begin{table}[!htbp]
\centering
\small
\setlength{\tabcolsep}{4.5pt}
\input{tables/shift_audits}
\caption{Descriptive audits without row correspondence. AUC is diagnostic;
Brier and Cost5x entries are mean proxy-minus-source regrets.}
\label{tab:shifts}
\end{table}

\Cref{tab:shifts} shows that Bank meets the registered relative-transfer
requirements although AUC falls by 0.030; Brier and normalized Cost5x risks
fall by 0.135 and 0.317. Heart remains unresolved with 323 target records.
Adult's reduced schema produces a corrected Brier-regret violation. These
labels apply only to the registered relative requirements and do not inherit
the prospective guarantee because these public records had already informed
the project.

\section{Learning procedure}
\label{app:learning_procedure}

The source and proxy use the same three-pipeline library implemented with
scikit-learn \citep{pedregosa2011scikit}. Logistic regression uses median
imputation, standardization, $C=1$, and at most 2{,}000 solver iterations.
The random forest has 300 trees and a minimum leaf size of three. Histogram
boosting uses learning rate 0.05, 200 iterations, and at most 31 leaves.
For each fitted pipeline, we minimize validation Cost5x over 501 thresholds
between zero and one. We fix the lowest-cost pipeline and its threshold
before opening the target audit. Source and proxy make this choice
independently.
Generator and transformation protocols are in
\cref{app:generator_details,app:transformation_screen}.

\section{Generator details}
\label{app:generator_details}

\paragraph{AIM.}
We treat feature ranges and category domains as public benchmark schema.
They are not estimated as private quantities. The private training split is
60\% of each dataset, validation and target audit splits are 20\% each, and
releases have the same row count as the private training split. AIM uses
degree two, maximum model size 40, 40 marginals, at most 5{,}000 cells, ten
bins, $\delta=10^{-9}$, and five registered release seeds per privacy cell.
The Heart, German, and Taiwan target audits contain 54, 200, and 6{,}000
records \citep{janosi1989heart,hofmann1994germanuci,yeh2009defaultuci}.

\paragraph{TabDDPM.}
We use the official Yandex Research implementation \citep{kotelnikov2023tabddpm}.
It models all encoded benchmark features numerically using standard normalization. The MLP has
three hidden layers of width 256 and time-embedding dimension 128. Training
uses 5{,}000 steps, batch size 4{,}096, learning rate $10^{-3}$, and 100
diffusion steps. The run used Python 3.10, PyTorch 2.0.1 with CUDA 11.7, and
an NVIDIA GTX 1650. We fixed these settings before the full run.
\Cref{tab:tabddpm} reports the three resulting release audits.

\begin{table}[!htbp]
\centering
\small
\setlength{\tabcolsep}{4.5pt}
\input{tables/tabddpm_audit}
\caption{Realized-release TabDDPM audit. Brier and Cost5x are mean regrets;
AUC remains diagnostic.}
\label{tab:tabddpm}
\end{table}

\begingroup
\raggedbottom
\section{Attack construction}
\label{app:attack}

For attack $k$, we reserve nonmembers to set a score threshold $c_k$ at a
registered false-positive rate. On a uniform subsample drawn without
replacement from the remaining pools, define
\[
D_{i,Q,k}^{\mathrm{attack}}
=
\mathbf{1}\{s_k(X_i^{\mathrm{member}})\geq c_k\}
-
\mathbf{1}\{s_k(X_i^{\mathrm{nonmember}})\geq c_k\}.
\]
Its finite-pool mean is attack advantage and lies in $[-1,1]$. Hoeffding's
inequality remains valid for uniform sampling without replacement
\citep{serfling1974probability}. Our registered ceiling is 0.05, an
illustrative policy setting rather than a universal definition of safety.

The attack representation uses median imputation, standardization, and at
most ten principal components. We fit it on reference records only.
DOMIAS-KDE is the log density ratio between KDEs fitted to the synthetic and
reference tables. Gen-LRA uses a Gaussian KDE with Silverman's bandwidth and
the ten nearest synthetic neighbors. The density-only attack drops the
reference denominator. DCR is negative distance to the closest synthetic
record.

Each attacker seed draws at most 500 members, 500 nonmembers, 1{,}000
reference records, and 2{,}000 synthetic records without replacement. The
reported intervals use 500 stratified bootstrap repetitions within each
seed. These bootstrap intervals describe score sensitivity and do not enter
the Holm decision.

For the finite-pool advantage test, half of the nonmember sample calibrates
the score threshold. The remaining nonmembers and an equally sized member
sample form bounded detection differences. Seed 911 and an advantage ceiling
of 0.05 remain fixed in the registry.
\Cref{tab:privacy_attacks} summarizes the corrected claim decisions, and
\cref{fig:attack_heatmap} shows the descriptive AUCs across the registered
subsamples.

\begingroup
\begin{table}[H]
\centering
\small
\setlength{\tabcolsep}{3.5pt}
\input{tables/privacy_attacks}
\caption{No-box membership results, aggregated by generator and dataset.
We select the highest AUC only for display. The final column counts
validated/unresolved/violation decisions for the registered attack claim.}
\label{tab:privacy_attacks}
\end{table}

\begin{figure}[H]
\centering
\includegraphics[width=0.94\textwidth]{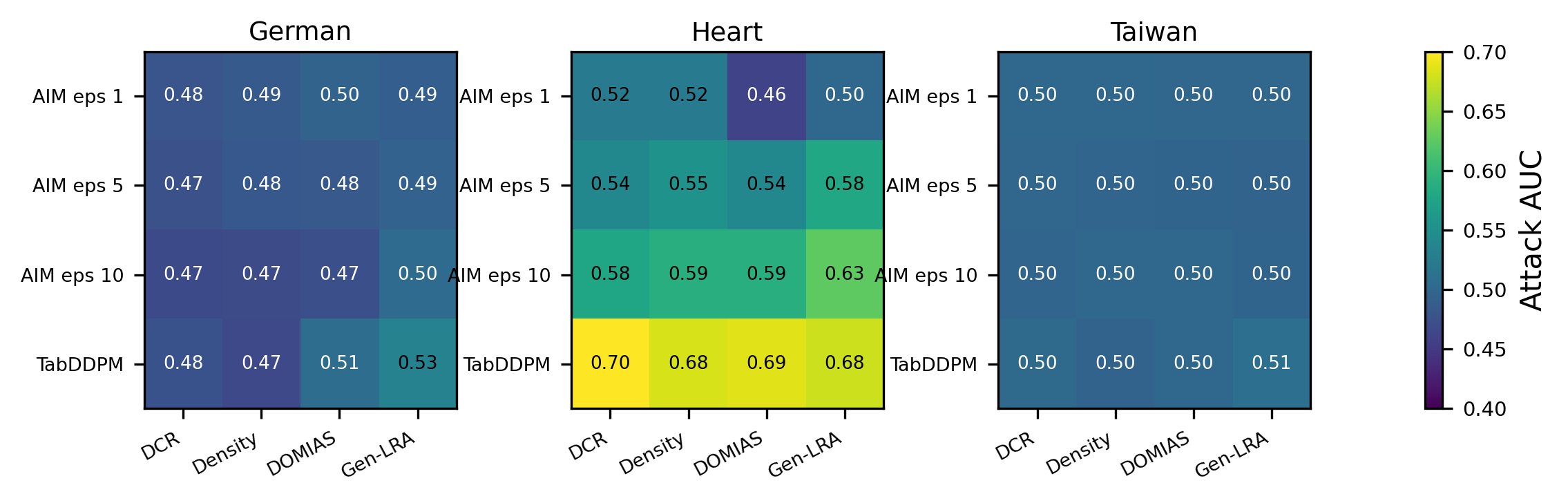}
\caption{Mean membership AUC over ten registered subsamples. Rows are
realized releases; Density denotes density-only KDE. Values near 0.5
indicate little discrimination for that attack, not a general privacy
guarantee.}
\label{fig:attack_heatmap}
\end{figure}
\endgroup

\section{Public transformation screen}
\label{app:transformation_screen}

The earlier transformation screen remains a useful sample-size check. It
uses one stratified 60/20/20 split on ten public binary-classification
datasets from UCI, Kaggle's Give Me Some Credit competition, and ProPublica's
COMPAS analysis \citep{uci2023repository,kaggle2011gmsc,larson2016compas}.
We fix logistic regression and temperature scaling without target access.
We compare each source with 10\% numeric noise, quartile coarsening, and
20\% feature masking, giving 30 candidates in one Holm family.
\Cref{tab:real_audit_full} reports the resulting audit states.

\begingroup
\begin{table}[H]
\centering
\small
\renewcommand{\arraystretch}{1.05}
\begin{tabular}{lrrl}
\toprule
Dataset & Audit $n$ & Validated / 3 & Remaining state \\
\midrule
Adult Income & 6{,}513 & 1 & 2 unresolved \\
Australian Credit & 138 & 0 & 3 unresolved \\
Bank Marketing & 9{,}043 & 2 & 1 unresolved \\
Breast Cancer WDBC & 114 & 0 & 3 unresolved \\
COMPAS & 1{,}443 & 0 & 3 unresolved \\
German Credit & 200 & 0 & 3 unresolved \\
Give Me Some Credit & 30{,}000 & 3 & none \\
Heart Disease & 54 & 0 & 3 unresolved \\
Mammographic Mass & 193 & 0 & 3 unresolved \\
Taiwan Default & 6{,}000 & 1 & 2 unresolved \\
\bottomrule
\end{tabular}
\caption{Risk-controlled transformation results. The audit detects no
corrected violation.}
\label{tab:real_audit_full}
\end{table}
\endgroup

\endgroup

\section{Prospective 20 Newsgroups text audit}
\label{app:text_direct}

This project-new audit uses the hockey and space classes from 20 Newsgroups
\citep{lang1995newsweeder}. Development loaded only the official training
subset. A registered split assigned 895 records to source fitting and 298 to
planning; a 500-word source-only vocabulary and MultinomialNB procedure were
fixed through scikit-learn \citep{pedregosa2011scikit}. The frozen
class-conditional unigram model generates labeled count-vector releases. Each
mechanism draw samples a new synthetic corpus and refits the proxy classifier;
it does not refit the unigram generator.

The primary limits equal source planning risk plus registered degradation
budgets of 0.04 for error and normalized 3:1 decision cost and 0.03 for Brier
loss. Registered absolute ceilings cap them. The resulting limits are 0.1205,
0.0789, and 0.0736; the score cutoffs are 0.0905, 0.0589, and 0.0436. Strict
and permissive envelopes were also fixed, but the primary specification alone
determines the headline decision.

Development-only evidence selected high-signal, moderate-evidence, and degraded
configurations by a deterministic margin rule. The original 2{,}000-record
target plan was then found infeasible even under all favorable scores. Before
the test subset was loaded, a recorded amendment increased the i.i.d.\ target
draw to 4{,}000 and changed nothing else. The target population is the empirical
distribution over the 793 official test records, not an unrestricted stream of
future newsgroup posts. We sampled it with replacement using the registered
seed. We generated and froze all 1{,}000 release procedures before this
target was opened.

\begin{table}[H]
\centering
\small
\setlength{\tabcolsep}{4pt}
\input{tables/text_direct_audit}
\caption{Prospective non-tabular text audit at $\eta_0=0.95$. $V$ is the
number of named recognitions. Bounds are simultaneous over the three registered
configurations.}
\label{tab:text_direct_audit}
\end{table}

\Cref{tab:text_direct_audit} shows that all high and moderate releases are
named, so this audit does not reproduce the
multiplicity-limited direct-only regime. The degraded configuration is
unresolved rather than declared invalid. The expectation-only contamination
allowance is 0.000414. The smooth certificate is vacuous here because its
three-factor target-sensitivity allowance reaches one, illustrating that added
smooth structure does not guarantee a tighter certificate. Release generation
took 35.0 seconds and sealed-target evaluation took 30.7 seconds on the audit
machine. This study broadens the application beyond tabular records but uses a
simple bag-of-words generator, not a modern language model.

\section{Reproducibility}
\label{app:reproducibility}

The accompanying code repository provides frozen study protocols and amendments,
exact dependency versions, public-data acquisition instructions, and commands
for each experiment. Its documentation covers preparation, development pilots,
protocol freezing, release generation, and sealed-target evaluation. Software
revisions and integrity records are kept with the experimental artifacts.

The sealed MAGIC, Spambase, Mushroom, CDC, Covertype, Online Shoppers, Rice,
and 20 Newsgroups study records identify the development data, audit reserves,
and release bundles fixed before target access. A dated lineage note documents
the post hoc correction for earlier studies that reused target records.

The current prediction-audit code writes original source-row positions,
labels, source and proxy probabilities, selected thresholds, component
$p$-values, and simultaneous bounds. The lineage script reconstructs the
historical pilot and rerun roles from registered seeds and writes both a
summary and a record-level mapping. The risk-control code reads per-record
losses rather than reconstructing evidence from aggregate metrics.

%% file: tables/cost_planning.tex
\begin{tabular}{lrrrrr}
\toprule
Regime & $c_Q/c_Z$ & Mode & $n$ & $R$ & Projected lower bound \\
\midrule
Moderate & 0.01 & direct & 5,000 & 1,000 & 0.935 \\
Moderate & 0.10 & direct & 5,000 & 1,000 & 0.935 \\
Moderate & 1 & direct & 5,000 & 1,000 & 0.935 \\
Moderate & 10 & direct & 5,000 & 500 & 0.927 \\
Moderate & 100 & direct & 5,000 & 30 & 0.779 \\
High signal & 0.01 & named & 250 & 1,000 & 0.969 \\
High signal & 0.10 & named & 250 & 1,000 & 0.969 \\
High signal & 1 & named & 250 & 1,000 & 0.969 \\
High signal & 10 & named & 250 & 500 & 0.964 \\
High signal & 100 & named & 250 & 30 & 0.828 \\
\bottomrule
\end{tabular}

%% file: tables/cost_planning_sensitivity.tex
\begin{tabular}{rrrr}
\toprule
$c_Q/c_Z$ & Mode switch (\%) & Mean regret & Maximum regret \\
\midrule
0.01 & 6.1 & 0.010 & 0.161 \\
0.10 & 6.1 & 0.010 & 0.161 \\
1 & 6.1 & 0.010 & 0.161 \\
10 & 6.1 & 0.010 & 0.165 \\
100 & 8.2 & 0.013 & 0.165 \\
\bottomrule
\end{tabular}

%% file: tables/target_reuse.tex
\begin{tabular}{lrrr}
\toprule
Audit design & False validations & Rate & 95\% Wilson interval \\
\midrule
Select, then reuse target & 2,173/5,000 & 0.435 & [0.421, 0.448] \\
Select, then seal target & 154/5,000 & 0.031 & [0.026, 0.036] \\
Test full family with Holm & 118/5,000 & 0.024 & [0.020, 0.028] \\
\bottomrule
\end{tabular}

%% file: tables/out_of_mechanism.tex
\begin{tabular}{lrrrr}
\toprule
Held-out regret family & Point FWER & Uncorrected FWER & \method{} FWER & \method{} power \\
\midrule
Continuous beta & 0.985 & 0.000 & 0.000 & 0.841 \\
Rare-subgroup mixture & 0.987 & 0.000 & 0.000 & 0.000 \\
Correlated candidates & 0.703 & 0.000 & 0.000 & 0.924 \\
\bottomrule
\end{tabular}

%% file: tables/stratified_subgroup.tex
\begin{tabular}{rrrrrrr}
\toprule
Audit $n$ & Random rare $n$ & Random FWER & Random power &
Strat.\ rare $n$ & Strat.\ FWER & Strat.\ power \\
\midrule
500  & 10.0 & 0.000 & 0.000 & 250  & 0.029 & 0.969 \\
1{,}000 & 20.0 & 0.000 & 0.000 & 500  & 0.049 & 1.000 \\
2{,}500 & 50.0 & 0.018 & 0.090 & 1{,}000 & 0.047 & 1.000 \\
\bottomrule
\end{tabular}

%% file: tables/appendix_experiment_roadmap.tex
\begin{tabularx}{\textwidth}{@{}lXcc@{}}
\toprule
Study block & Purpose & Confirmatory? & Shared target? \\
\midrule
Registered simulation & baseline mechanism calibration & yes & no \\
Initial direct study & direct-mode sanity check & yes & yes \\
Primary comparison & direct versus named-release inference & yes & yes \\
False-pass diagnostic & necessity of the contamination term & yes & yes \\
Smooth concentration & structured target-side correction & yes & yes \\
Collective grid & independent-batch baseline & yes/amended & no \\
Planning studies & multiplicity and audit design & yes & varies \\
Historical studies & mechanism sensitivities & mixed & yes \\
\bottomrule
\end{tabularx}

%% file: tables/mechanism_calibration.tex
\begin{tabular}{lrr}
\toprule
Method & False mechanism validation & Power \\
\midrule
Plug-in fraction & 0.997 & 1.000 \\
Point rule + binomial & 0.059 & 1.000 \\
Per-release IUT + binomial & 0.006 & 0.996 \\
Two-level \method{} & 0.006 & 0.996 \\
Collective partial-conjunction \method{} & 0.006 & 0.996 \\
Oracle release labels & 0.025 & 1.000 \\
\bottomrule
\end{tabular}

%% file: tables/conditional_shared_target.tex
\begin{tabular}{@{}lccc@{}}
\toprule
True reliability & Named Holm & Direct shared-target & Oracle labels \\
\midrule
$0.80$ & $2.2 \pm 0.5$ & $0/1{,}000$ & $4.6 \pm 0.7$ \\
$0.90$ & $79.9 \pm 1.3$ & $69.0 \pm 1.5$ & $1{,}000/1{,}000$ \\
$0.95$ & $80.1 \pm 1.3$ & $\mathbf{89.2 \pm 1.0}$ & $1{,}000/1{,}000$ \\
\bottomrule
\end{tabular}

%% file: tables/false_pass_diagnostic.tex
\begin{tabular}{@{}lrrr@{}}
\toprule
Method & False validation (\%) & Events & Mean lower bound \\
\midrule
Score-only, no correction & 8.05 & 161/2{,}000 & 0.741 \\
Direct shared-target & 0.00 & 0/2{,}000 & 0.000 \\
Oracle release labels & 4.00 & 80/2{,}000 & 0.754 \\
\bottomrule
\end{tabular}

%% file: tables/smooth_target_concentration.tex
\begin{tabular}{lrrr}
\toprule
Method & $\eta=0.80$ & $\eta=0.95$ & $\eta=0.99$ \\
\midrule
Hard score + Markov & 0.00 & 0.00 & 0.00 \\
Ramp score + Markov & 0.00 & 0.00 & 0.00 \\
Ramp + target concentration & 0.00 & 0.00 & 86.45 \\
Oracle release labels & 4.25 & 100.00 & 100.00 \\
\bottomrule
\end{tabular}

%% file: tables/collective_benchmarks.tex
\begin{tabular}{lrrrr}
\toprule
Mechanism input & Boundary FWER & Power & Mean count & Gain over Holm \\
\midrule
Named-release Holm & 0.0000 (0.0000) & 0.1266 (0.0039) & 83.29 & 0.00 \\
Tail-Simes & 0.0005 (0.0005) & 0.6501 (0.0055) & 90.69 & 7.40 \\
Tail-Fisher & 0.0005 (0.0005) & 0.9035 (0.0036) & 92.71 & 9.42 \\
\bottomrule
\end{tabular}

%% file: tables/near_boundary_inner_correction.tex
\begin{tabular}{llrr}
\toprule
Bad-release risk & Inner rule & Any false release label & False mechanism validation \\
\midrule
0.100 & Separate release IUTs & 0.9765 & 0.0237 \\
0.100 & Inner Holm & 0.0131 & 0.0170 \\
0.101 & Separate release IUTs & 0.9520 & 0.0233 \\
0.101 & Inner Holm & 0.0113 & 0.0171 \\
0.102 & Separate release IUTs & 0.9325 & 0.0237 \\
0.102 & Inner Holm & 0.0079 & 0.0188 \\
0.120 & Separate release IUTs & 0.0600 & 0.0169 \\
0.120 & Inner Holm & 0.0001 & 0.0169 \\
\bottomrule
\end{tabular}

%% file: tables/alpha_allocation.tex
\begin{tabular}{rrrrr}
\toprule
$\lambda$ & $\alpha_{\mathrm R}$ & $\alpha_{\mathrm M}$ & False mechanism validation & Power \\
\midrule
0.10 & 0.005 & 0.045 & 0.020 & 0.645 \\
0.25 & 0.013 & 0.037 & 0.019 & 0.643 \\
0.50 & 0.025 & 0.025 & 0.021 & 0.602 \\
0.75 & 0.037 & 0.012 & 0.000 & 0.000 \\
0.90 & 0.045 & 0.005 & 0.000 & 0.000 \\
\bottomrule
\end{tabular}

%% file: tables/mechanism_planning.tex
\begin{tabular}{rrr}
\toprule
Reliability target & Mechanisms screened & All-recognized releases needed \\
\midrule
80\% & 1 & 17 \\
80\% & 5 & 24 \\
80\% & 9 & 27 \\
90\% & 1 & 36 \\
90\% & 5 & 51 \\
90\% & 9 & 56 \\
\bottomrule
\end{tabular}

%% file: tables/adaptive_search.tex
\begin{tabular}{rrrr}
\toprule
Rounds & Fixed-$\alpha$ false validation & Spending false validation & Spending power \\
\midrule
1 & 0.034 & 0.020 & 1.000 \\
10 & 0.271 & 0.030 & 0.914 \\
25 & 0.549 & 0.031 & 0.813 \\
50 & 0.798 & 0.032 & 0.651 \\
\bottomrule
\end{tabular}

%% file: tables/aim_mechanism_reliability.tex
\begin{tabular}{lrrrrl}
\toprule
Dataset & $\epsilon$ & Releases & Release V/U/F & One-sided bounds & Mechanism decision \\
\midrule
German & 1 & 5 & 0/5/0 & [0.00, 1.00] & Unresolved \\
German & 5 & 5 & 0/5/0 & [0.00, 1.00] & Unresolved \\
German & 10 & 5 & 0/5/0 & [0.00, 1.00] & Unresolved \\
Heart & 1 & 5 & 0/5/0 & [0.00, 1.00] & Unresolved \\
Heart & 5 & 5 & 0/5/0 & [0.00, 1.00] & Unresolved \\
Heart & 10 & 5 & 0/5/0 & [0.00, 1.00] & Unresolved \\
Taiwan & 1 & 5 & 0/0/5 & [0.00, 0.69] & Violation \\
Taiwan & 5 & 5 & 0/0/5 & [0.00, 0.69] & Violation \\
Taiwan & 10 & 5 & 0/1/4 & [0.00, 0.84] & Unresolved \\
\bottomrule
\end{tabular}

%% file: tables/prospective_aim_mechanism.tex
\begin{tabular}{lrrrrl}
\toprule
Mechanism & Releases & Release V/U/F & One-sided bounds & Mean Cost5x regret & Decision \\
\midrule
AIM, $\epsilon=1$ & 24 & 0/0/24 & [0.00, 0.14] & +0.035 & Violation \\
\bottomrule
\end{tabular}

%% file: tables/positive_bootstrap_requirements.tex
\begin{tabular}{llrrr}
\toprule
Requirement & Form & Limit & Mean & Largest upper bound \\
\midrule
Proxy Brier risk & Absolute risk & 0.180 & 0.138 & 0.157 \\
Proxy Cost5x risk & Absolute risk & 0.160 & 0.113 & 0.136 \\
Proxy log-loss risk & Absolute risk & 0.070 & 0.032 & 0.039 \\
Brier transfer & Relative regret & 0.040 & 0.004 & 0.017 \\
Cost5x transfer & Relative regret & 0.040 & 0.004 & 0.027 \\
Log-loss transfer & Relative regret & 0.040 & 0.001 & 0.010 \\
\bottomrule
\end{tabular}

%% file: tables/positive_bootstrap_baselines.tex
\begin{tabular}{lrrr}
\toprule
Reference & Brier & Clipped log loss & Cost5x \\
\midrule
Source procedure & 0.135 & 0.031 & 0.109 \\
Constant 0.5 & 0.250 & 0.050 & 0.156 \\
Training prevalence & 0.172 & 0.038 & 0.156 \\
\midrule
Registered ceiling & 0.180 & 0.070 & 0.160 \\
\bottomrule
\end{tabular}

%% file: tables/audit_lineage.tex
\begin{tabular}{lrrrr}
\toprule
Study & Audit $n$ & In pilot train & In pilot validation & In pilot audit \\
\midrule
AIM informed rerun & 6{,}000 & 3{,}623 & 1{,}169 & 1{,}208 \\
Bootstrap informed rerun & 6{,}000 & 3{,}598 & 1{,}184 & 1{,}218 \\
\bottomrule
\end{tabular}

%% file: tables/claim_status.tex
\begin{tabularx}{\textwidth}{@{}>{\raggedright\arraybackslash}p{2.8cm}>{\raggedright\arraybackslash}p{2.5cm}>{\raggedright\arraybackslash}X>{\raggedright\arraybackslash}p{1.45cm}>{\raggedright\arraybackslash}p{1.35cm}@{}}
\toprule
Experiment & Evidence status & Target lineage & Registered & Guarantee \\
\midrule
Release calibration and target-reuse studies & Confirmatory simulation & Fresh simulated audit draws & Yes & Yes \\
Mechanism calibration and planning studies & Confirmatory simulation & Fresh simulated audit draws & Yes & Yes \\
Direct and smooth shared-target confirmations & Confirmatory simulations & Fresh simulated targets shared across releases & Yes & Yes \\
Adaptive candidate stream & Confirmatory simulation & Fresh simulated batch each round & Yes & Yes \\
Temporal site and schema proxies & Descriptive real-data audit & Public benchmarks used earlier in project & Full-run settings only & No \\
Repeated AIM releases & Retrospective sensitivity & Existing public benchmark splits and releases & No & No \\
Initial bootstrap control & Registered same-table run & Taiwan records used in earlier project analyses & Yes & No \\
Informed AIM and bootstrap reruns & Same-table sensitivity & Every audit record appeared in pilot split & Partly & No \\
TabDDPM releases & Registered descriptive audit & Public benchmark targets used earlier in project & Yes & No \\
Public transformation screen & Descriptive stress screen & Public benchmark targets used earlier in project & No & No \\
Membership attack suite & Registered descriptive audit & Developed releases and public record pools & Yes & No \\
MAGIC sealed bootstrap mechanism & Prospective real mechanism audit & Uniform, record-disjoint unlabeled reserve from project-new dataset & Yes & Yes \\
Spambase sealed AIM mechanism & Prospective private mechanism audit & Record-disjoint unlabeled reserve from project-new dataset & Yes; pre-audit amendment & Yes \\
Secondary Mushroom private conditional sampler & Prospective direct mechanism audit & Uniform 10{,}000-record reserve fixed before outcomes were parsed & Yes; two pre-audit amendments & Yes \\
CDC neural configurations & Corrective stratified analysis of a prospective audit & Fixed class strata; equal-mixture target & Protocol yes; sampling correction after review & Yes for corrected analysis \\
Covertype neural configurations & Corrective stratified analysis of a prospective audit & Fixed class strata; equal-mixture target & Protocol yes; sampling correction after review & Yes for corrected analysis \\
Online Shoppers neural configurations & Corrective stratified analysis of a prospective audit & Fixed class strata; equal-mixture target & Protocol yes; sampling correction after review & Yes for corrected analysis \\
Rice--TVAE configurations & Prospective full-pipeline standard-generator audit & I.i.d.\ target sample from a record-disjoint project-new population & Yes; pre-target compute amendments & Yes \\
20 Newsgroups text configurations & Prospective non-tabular mechanism audit & Official test subset first loaded after freeze; i.i.d.\ empirical-target draw & Yes; pre-target size amendment & Yes \\
\bottomrule
\end{tabularx}

%% file: tables/magic_sealed_requirements.tex
\begin{tabular}{lrrrrrr}
\toprule
Metric & Rel. mean & Rel. max UCB & Limit & Abs. mean & Abs. max UCB & Limit \\
\midrule
Brier & 0.0031 & 0.0158 & 0.040 & 0.0973 & 0.1143 & 0.141 \\
Clipped log loss & 0.0006 & 0.0100 & 0.040 & 0.0232 & 0.0301 & 0.072 \\
Cost5x & 0.0027 & 0.0216 & 0.040 & 0.0675 & 0.0871 & 0.108 \\
\bottomrule
\end{tabular}

%% file: tables/spambase_aim_requirements.tex
\begin{tabular}{lrrrrrr}
\toprule
Metric & Rel. limit & Rel. mean & Viol. & Abs. limit & Abs. mean & Viol. \\
\midrule
Brier & 0.060 & 0.1467 & 30/30 & 0.174 & 0.2558 & 30/30 \\
Clipped log loss & 0.060 & 0.0356 & 0/30 & 0.080 & 0.0621 & 0/30 \\
Cost5x & 0.060 & 0.0592 & 1/30 & 0.091 & 0.0998 & 1/30 \\
\bottomrule
\end{tabular}

%% file: tables/neural_full_pipeline_audits.tex
\begin{tabular}{llrrrrl}
\toprule
Dataset & Configuration & $R$ & Named $V/W$ & Scores & $L_{\rm named}$ & $L_{\rm dir}$ \\
\midrule
CDC & high & 100 & 100/0 & 99 & 0.938 & 0.917 \\
CDC & moderate & 300 & 295/0 & 292 & 0.948 & 0.931 \\
CDC & degraded & 100 & 0/85 & 0 & 0 & 0 \\
Covertype & high & 100 & 100/0 & 100 & 0.938 & 0.910 \\
Covertype & moderate & 300 & 295/0 & 288 & 0.948 & 0.885 \\
Covertype & degraded & 100 & 0/100 & 0 & 0 & 0 \\
Online Shoppers & high & 100 & 60/0 & 52 & 0.449 & 0.363 \\
Online Shoppers & moderate & 500 & 5/0 & 30 & 0.002 & 0.017 \\
Online Shoppers & degraded & 100 & 0/44 & 0 & 0 & 0 \\
\bottomrule
\end{tabular}

%% file: tables/rice_tvae_audit.tex
\begin{tabular}{lrrrrrl}
\toprule
Configuration & $R$ & Named $V/W$ & Scores & $L_{\rm named}$ & $L_{\rm dir}$ & Decision \\
\midrule
High fidelity (200 epochs) & 40 & 40/0 & 40 & 0.852 & 0.864 & both validate \\
Moderate (80 epochs) & 120 & 65/1 & 67 & 0.405 & 0.428 & unresolved \\
Degraded & 40 & 2/34 & 2 & 0.001 & 0 & violation ($U=0.375$) \\
\bottomrule
\end{tabular}

%% file: tables/secondary_mushroom_requirements.tex
\begin{tabular}{lrrrr}
\toprule
Requirement & Score cutoff & Validity limit & Mean & Range across releases \\
\midrule
Brier transfer & 0.060 & 0.100 & 0.0474 & [0.0433, 0.0519] \\
Absolute Brier & 0.162 & 0.182 & 0.1467 & [0.1427, 0.1513] \\
Cost5x transfer & 0.041 & 0.081 & 0.0225 & [0.0165, 0.0320] \\
Absolute Cost5x & 0.097 & 0.117 & 0.0794 & [0.0733, 0.0889] \\
\bottomrule
\end{tabular}

%% file: tables/shift_audits.tex
\begin{tabular}{lrrrrl}
\toprule
Proxy relation & $n$ & $\Delta$AUC & Brier & Cost5x & Decision \\
\midrule
Bank Marketing temporal & 2,620 & -0.030 & -0.135 & -0.317 & Rel. validated \\
Heart Disease site & 323 & -0.062 & +0.014 & +0.104 & Unresolved \\
Adult reduced schema & 8,141 & -0.044 & +0.030 & +0.012 & Violation \\
\bottomrule
\end{tabular}

%% file: tables/tabddpm_audit.tex
\begin{tabular}{lrrrrl}
\toprule
Dataset & $n$ & $\Delta$AUC & Brier & Cost5x & Decision \\
\midrule
Heart & 54 & +0.015 & -0.015 & -0.011 & Unresolved \\
German & 200 & -0.035 & +0.011 & +0.021 & Unresolved \\
Taiwan & 6,000 & -0.044 & +0.037 & +0.010 & Violation \\
\bottomrule
\end{tabular}

%% file: tables/privacy_attacks.tex
\begin{tabular}{llrrrl}
\toprule
Generator & Dataset & Releases & Highest AUC & TPR@1\% & Decision counts \\
\midrule
AIM & German & 3 & 0.505 (Gen-LRA) & 0.023 & 0/3/0 \\
AIM & Heart & 3 & 0.626 (Gen-LRA) & 0.093 & 0/3/0 \\
AIM & Taiwan & 3 & 0.502 (density-only KDE) & 0.008 & 0/3/0 \\
TabDDPM & German & 1 & 0.533 (Gen-LRA) & 0.016 & 0/1/0 \\
TabDDPM & Heart & 1 & 0.702 (DCR) & 0.346 & 0/1/0 \\
TabDDPM & Taiwan & 1 & 0.510 (Gen-LRA) & 0.016 & 0/1/0 \\
\bottomrule
\end{tabular}

%% file: tables/text_direct_audit.tex
\begin{tabular}{lrrrrl}
\toprule
Configuration & $R$ & Named $V$ & $L^{\mathrm{named}}$ & $L^{\mathrm{dir}}$ & Decision \\
\midrule
High signal & 250 & 250 & 0.9810 & 0.9743 & Both validate \\
Moderate evidence & 500 & 500 & 0.9905 & 0.9869 & Both validate \\
Degraded & 250 & 54 & 0.1569 & 0.0701 & Unresolved \\
\bottomrule
\end{tabular}